\def\forarxiv{1}

\if1\forarxiv
\documentclass[12pt]{article}
\usepackage{authblk}
\usepackage{subcaption} 
\fi
\usepackage[numbers, sort&compress]{natbib} %

\usepackage{graphicx}
\usepackage{multirow}
\usepackage{amsmath, amsthm, amssymb, amsthm, outlines,bm}
\usepackage{shortcuts,enumitem}
\usepackage[margin=.75in]{geometry}
\usepackage{booktabs}
\usepackage{caption}
\usepackage[usenames, dvipsnames]{color}
\usepackage{nicefrac}
\usepackage{appendix}

\definecolor{light-gray}{gray}{.35}

\if0\forarxiv
\usepackage{parskip}
\fi
\newcommand{\pr}{\mathbb{P}}
\newcommand{\expect}{\mathbb{E}}

\newcommand{\ind}{\mathbb{I}}
\newcommand{\avg}[1]{\frac{1}{n}\sum_{i = 1}^n}

\newcommand{\ie}{\textit{i.e.}}
\newcommand{\eg}{\textit{e.g.}}

\if0\forarxiv
\usepackage[compact,small]{titlesec}
\fi

\newtheorem{theorem}{Theorem}%[section]
\newtheorem{corollary}{Corollary}[theorem]
\newtheorem{lemma}{Lemma}%[section]
\newtheorem{proposition}{Proposition}%[section]

\newtheorem{definition}{Definition}%[section]

\let\citet\cite
\renewcommand{\paragraph}[1]{\textbf{#1}}

\begin{document}

%`
%`
%`
%`
%`
%`

%`
%`
%`
%`
%`
%`

\newcommand{\mytitle}{Interval Estimation of Individual-Level Causal Effects Under Unobserved Confounding}

\if1\forarxiv
\title{\mytitle}
\author{Nathan Kallus}
\author{Xiaojie Mao}
\author{Angela Zhou}
%`
%`
%`
\affil{Cornell University and Cornell Tech}
%`
\date{}
\maketitle
%`
\fi

\begin{abstract}
We study the problem of learning conditional average treatment
effects (CATE) from observational data with unobserved confounders.
The CATE function maps baseline covariates to individual
causal effect predictions and is key for personalized assessments.
Recent work has focused on how to learn CATE under unconfoundedness,
\ie, when there are no unobserved confounders.
Since CATE may not be identified when unconfoundedness
is violated, we develop a functional interval estimator that
predicts bounds on the individual causal effects under 
realistic violations of unconfoundedness. Our estimator takes the
form of a weighted kernel estimator with weights that vary 
adversarially. We prove that our estimator is sharp in that it
converges exactly to the tightest bounds possible on CATE when
there may be unobserved confounders. Further, we study personalized 
decision rules derived from our estimator and prove that they
achieve optimal minimax regret asymptotically.
We assess our approach in a simulation study as well as demonstrate
its application in the case of hormone replacement therapy by
comparing conclusions from a real observational study and 
clinical trial.
\end{abstract}

\section{Introduction}
 Learning individual-level (conditional-average) causal effects from observational data is a key question for determining personalized treatments in medicine or in assessing policy impacts in the social sciences. Many recent advances have been made for the important question of estimating {conditional average treatment effects} (CATE), which is a function mapping baseline covariates to individual causal effect predictions \citep{shalit-johansson-sontag-17,johansson-k-s-s18,athey2016recursive,wager2017estimation,abrevaya2015estimating,green2010modeling,kunzel2017meta,nie2017learning}. However, all of these approaches need to assume \textit{unconfoundedness}, or that the potential outcomes are conditionally independent of treatments, given observed covariates.
 That is, that all possible confounders have been observed and are controlled for.

 While {unconfoundedness} may hold \emph{by design} in ideal settings like randomized controlled trials, the assumption is almost always \emph{invalid to some degree} in any real observational study
 and, to make things worse, the assumption is inherently
 \emph{unverifiable}. 
 For example, passively collected healthcare databases
 often lack part of the critical clinical information that may drive both doctors' and patients' treatment choices, \eg, subjective assessments of condition severity or personal lifestyle factors. The expansion and linkage of observational and administrative datasets does afford greater opportunities to observe important factors that influence selection into treatment, but some hidden factors will always remain and there is no way to prove otherwise. 
 %`
 When unconfoundedness does not hold, 
 one might find causal effects in observational data where there is actually no real effect or vice versa, which in turn may lead to real decisions that dangerously rely on false conclusions and may introduce unnecessary harm or risk \citep{kz18}.
 
 Therefore, sensitivity analysis of causal estimates to realistic violations of unconfoundedness is crucial for both credible interpretation of any findings and reliable decision making.
 Traditional sensitivity analysis and modern extensions focus on  bounding feasible values of \textit{average treatment effects} (ATE) 
 or the corresponding $p$-values for the hypothesis of zero effect, assuming some violation of unconfoundedness
 \citep{r02,doriehchill16,zhaosmall17}. 

 However, ATE is of limited use for individual-level assessments and personalization. For such applications, it is crucial to study the \emph{heterogeneity} of effects as covariates vary by 
 estimating the CATE.
 Given an individual with certain baseline covariates, the sign of 
 the CATE for their specific values determines the best course
 of action for the individual.
 Furthermore, learning CATE also allows one to generalize causal conclusions drawn from one population to another population \citep{hartman2015sate}. 

 In this paper, we develop new methodology and theory for
 learning bounds on the CATE function from observational data
 that may be subject to some confounding. 
 Our contributions are summarized as follows:
 \if1\forarxiv
 \begin{itemize}
 \fi\if0\forarxiv
 \vspace{-0.5\baselineskip}
 \begin{itemize}[align=left, leftmargin=*,labelindent=0in,itemsep=0in,parsep=0in]
 \fi
 \item We propose a functional interval estimator for CATE 
 that is derived from a 
 weighted kernel regression, where the weights vary adversarially per 
 a standard sensitivity model that specifies the how big the potential 
 impact of hidden confounders might be on selection. We extend the approach to conditioning on a subset of observed covariates
 (partial CATE).
 \item We show that the proposed estimators, which are given by an optimization problem, admit efficient computation by a sorting and line search procedure.
 \item We show that our estimator is \emph{sharp} in that it converges 
 point-wise to the tightest possible set of identifiable
 CATE functions -- the set of CATE functions 
 that are consistent with
 both the population-level observational-data generating process 
 and the assumed sensitivity model. 
 That is, our interval function is asymptotically
 neither too wide (too conservative) nor too narrow (too optimistic).
 \item We study personalized decision rules derived from our estimator
 %`
 %`
 and show that their minimax-regret converges to the best possible
 under the assumed sensitivity model.
 %`
 %`
 \item We assess the success of our approach in a simulation study as well as demonstrate the application of our approach in real-data setting. Specifically, we consider the individual-level effects of hormone replacement therapy and compare insights derived from a likely-confounded observational study to those derived from a clinical trial.
\end{itemize}
%`
 
\section{Related work}
\paragraph{Learning CATE.} Studying heterogeneous treatment effects by learning a functional form for CATE
under the assumption of unconfoundedness is a common approach 
 %`
 \citep{shalit-johansson-sontag-17,johansson-k-s-s18,athey2016recursive,wager2017estimation,abrevaya2015estimating,green2010modeling,kunzel2017meta,nie2017learning}.
 Under unconfoundedness, CATE is given by the difference of two
 identifiable regressions and the above work study how to appropriately
 tailor supervised learning algorithms specifically to such a task.
 In particular, \citet{abrevaya2015estimating} consider 
 estimating CATE under unconfoundedness using 
 kernel regression on a transformation of the outcome given
 by inverse propensity weighting (IPW).
 Our bounds arise from adversarially-weighted kernel 
 regression estimators, but in order to ensure sharpness,
 the estimators we use reweight the sample rather than the outcome.

 %`
 %`
 %`

 %`

 %`

\paragraph{Sensitivity analysis and partial identification.}
Sensitivity analysis in causal inference 
considers how the potential presence of unobserved confounders might
affect a conclusion made under the assumption of unconfoundedness.
It originated in a thought experiment on the effects of smoking
in lung cancer
that argued that unobservable confounding effects must be 
unrealistically large in order
to refute the observational evidence \citep{cornfield1959smoking}. 
In our approach,
we use the marginal sensitivity model (MSM) introduced by \cite{tan}, which bounds the potential impact of unobserved confounding on selection into treatment. Specifically, it bounds the ratio between
the propensity for treatment when accounting only for observables
and when accounting also for unobservables.
This model is closely related to the Rosenbaum sensitivity model \citep{r02}, which is traditionally used in conjunction with matching 
and which
bounds the ratio between the propensity for treatment between any two realization of unobservables.
See \cite{zhaosmall17} for more on the relationship between these two sensitivity models.
\citep{tan,zhaosmall17} consider the sensitivity of ATE estimates
under the MSM but not sharpness.
\cite{aronowlee12} consider a related problem of bounding
a population average under observations with 
unknown but bounded sampling probabilities and prove sharpness 
assuming discrete outcomes.
Instead of relying on a sensitivity model,
\cite{manski2003partial} consider sharp partial identification of 
ATE under no or weak assumptions such as monotone response and
\citet{manski05} consider corresponding minimax-regret policy
choice with discrete covariates. 
\citet{mp18} consider sharp partial identification of ATE under
knowledge of the marginal distribution of confounders and sup-norm
bounds on propensity differences.
%`

  %`
  %`
  %`
  %`

\paragraph{Personalized decision making.} 
Optimal personalized decision rules are given by thresholding 
CATE if known and hence a natural approach to learning such policies is 
to threshold CATE estimates \citep{qian2011performance}.
This of course breaks down if CATE is not estimable. Our paper
derives the appropriate extension to partially identified CATE
and shows that decision rules
derived from our estimates in fact achieve optimal minimax regret.
%`
Under unconfoundedness, recent work has also studied
directly learning a \emph{structured} decision policy 
from observational data since
best-in-class CATE estimates (\eg, best linear prediction) 
need not necessarily lead to best-in-class policies (\eg, best
linear policy) \citep{dell2014,sj15,athey2017efficient, kallus2017balanced,kallus2017off,kallus2016recursive,zhou2017residual}.
Recently, \cite{kz18} studied the problem of finding such structured policies that are also robust to possible confounding under a similar sensitivity model.
However, this approach produces only a policy and not 
a CATE estimate, which itself
is an important object for {decision support}
as one would like to interpret the policy relative to
predicted effects and understand
%`
the \emph{magnitude} of effect and 
\emph{uncertainties} in its estimation. 
%`
Causal effect estimates are of course also
important for influencing other conclusions, directing further the study of causal mechanisms,
and measuring conclusions against domain knowledge.

%`
 %`
%`

%`
%`
%`
%`
%`

\section{Problem Set-up} 
	We assume that the observational data consists of triples of random variables $\{(X_i,T_i,Y_i):i=1,\dots,n\}$, comprising of covariates $X_i \in \mathcal{X}\subseteq\R d$, assigned treatments $T_i \in \{0,1\}$, and real-valued outcomes $Y_i\in \mathcal{Y} \subseteq \mathbb R$. 
	Using the Neyman-Rubin potential outcome framework, we let $Y_i(0), Y_i(1) \in \mathcal{Y}$ denote the potential outcomes of each treatment.
	%`
	We let the observed outcome be the potential outcome of the assigned treatment, $Y_i=Y_i(T_i)$, encapsulating non-interference and consistency assumptions, also known as SUTVA \citep{rubin1980randomization}.
	Moreover, $(X_i,T_i,Y_i(0), Y_i(1))$ are i.i.d draws from a population $(X, T, Y(0), Y(1))$. 
	%`
	
	We are interested in the CATE function:
\begin{equation*}
	\tau(x)=\Efb{Y(1)-Y(0)\mid X=x}
\end{equation*}	
	\emph{If} $X$ contained all confounders, then we could identify CATE by \emph{controlling} for it in each treatment group: 
	%`
	%`
	%`
\begin{equation}
	\tilde\tau(x)=\Efb{Y\mid T=1,X=x}-\Efb{Y\mid T=0,X=x}. \label{eq: fake-cate}
\end{equation}
	\emph{If} unconfoundedness held in that potential outcomes are independent of assigned treatment given $X$, then it's immediate that $\tilde\tau$ would equal $\tau$.
	%`
	However, in practice there will almost always exist 
	\textit{unobserved confounders}
	{not} included in $X$, \ie, unconfoundedness with respect to $X$ is \emph{violated}. That is, in general, we may have that
	%`
			\begin{equation}
			%`
			\notag
				(Y(0), Y(1)) \notindependent T \mid X. 
			\end{equation}
	%`
	 %`
	 In such general settings,
	 $\tau(x)\neq\tilde\tau(x)$ and indeed $\tau(x)$ may \emph{not} be estimated from the observed data even with an infinite sample size \citep{r02}. 

	%`
	%`
	%`
	%`
	%`
	
	For $t \in \{0, 1\}$, let $e_t(x) = \pr (T = t \mid X = x)$ be the \textit{nominal} propensity for treatment given only the observed variables
	and $e_t(x, y) = \pr (T = t \mid X = x, Y(t) = y)$ be the \textit{complete} propensity accounting for all confounders. 
	%`
	%`
	%`
	%`
	%`
	%`
	%`
	%`
	%`
%`
	%`
	In this paper, we use the following sensitivity model to quantify the extent of violation of the unconfoundedness with respect to the observed covariates $X$. The model measures the degree of confounding in terms of the odds ratios of
	%`
	%`
	the nominal and complete propensities \citep{tan,zhaosmall17}.
	\begin{definition}[Marginal Sensitivity Model] \label{def: sensitivity-model}
		There exists $\Gamma \ge 1$ such that, for any $t \in \{0, 1\}, x \in \mathcal{X}, y \in \mathcal{Y}$, 
		\begin{equation}\label{eq: log-odds}
		%`
		\frac1\Gamma\leq
		{\frac{(1 - e_t(x))e_t(x, y)}{e_t(x)(1 - e_t(x, y))}}
		%`
		%`
		\leq
		%`
		\Gamma.
		\end{equation}
		%`
	\end{definition}\if0\forarxiv\vspace{-\baselineskip}\fi
	Taking logs, eq.~\eqref{eq: log-odds} can be seen as bounding
	the absolute difference between the logits
	the nominal propensity and the complete propensity by $\log\Gamma$.
	When unconfoundedness with respect to $X$ holds, we have that $e_t(x) = e_t(x, y) $ and \eqref{eq: log-odds} holds with $\Gamma = 1$.  When $\Gamma = 2$, for example, then the true odds ratio for an individual to be treated may be as much as double or as little as half of what it is actually observed to be given only $X$.
	%`
	As $\Gamma$ increasingly deviates from $1$, we allow for greater unobserved confounding.  
	%`
	%`
\section{An Interval Estimator for CATE}\label{sec: method}
%`
\subsection{Population estimands}\label{subsection: population-bounds}
We start by characterizing the population estimands we are after,
the population-level upper and lower bounds on CATE.
As discussed above, without unconfoundedness, there is no single CATE
function that can be point identified by the data.
Under the MSM with a given $\Gamma$, we can conceive of the set of
identified CATE functions as consisting of all the
functions that are consistent with
both the population of observational data $(X,T,Y)$ and the MSM.
All such functions are observationally equivalent in
that they cannot be distinguished from one another on the basis
of observational data alone. This defines a particular
interval function that maps covariates to the lower
and upper bounds of this set and this is the function we wish
to estimate.

For $t \in \{0, 1\}$ and $x \in \mathcal{X}$,
let $f_t(y \mid x)$ denote the density of the distribution $\pr(T = t, Y(t)\le y \mid X = x)=\pr(T = t, Y\le y \mid X = x)$. 
Note that these distributions are identifiable based on the observed data as they only involve observable quantities. Further, define $\mu_t(x) = \expect [Y(t) \mid X = x]$, which is \emph{not} identifiable from data, such that $\tau(x)=\mu_{1}(x)-\mu_{0}(x)$, and note that
%`
\begin{equation}
\mu_t(x)=\mu_t(w_t; x) = \frac{\int yw_t(y \mid x)f_t(y \mid x)dy}{\int w_t(y \mid x) f_t(y \mid x)dy}\label{eq: population-outcome}
\end{equation} 
%`
%`
where
$w_t(y \mid x) = 1/e_t(x, y)$, which too is unidentifiable.
%`
Eq. \eqref{eq: population-outcome} is useful as it
decomposes $\mu_t(x)$ cleanly into
its identifiable ($f_t(y \mid x)$) and unidentifiable ($w_t(y \mid x)$) 
components.
%`
%`
Based on the MSM \eqref{eq: log-odds}, we can determine the uncertainty set that includes all possible values of $w_t(y \mid x) =  1/e_t(x, y)$ that agree with the model, \ie, violate unconfoundedness by no more than $\Gamma$:
\begin{align}
&\mathcal{W}_t(x; \Gamma) =\braces{
w_t(\cdot\mid x):  w_t(y \mid x) \in [\alpha_t(x; \Gamma), \beta_t(x; \Gamma)]~\forall y} 
\notag
%`
%`
%`
%`
\\&\text{where}~~
\alpha_t(x; \Gamma) =1/(\Gamma e_t(x)) + 1 - 1/\Gamma, 
%`
\label{eq: uncertainty-set}
\\&\phantom{\text{where}~~}
\beta_t(x; \Gamma) =\Gamma/ e_t(x) + 1- \Gamma.  
\notag
%`
\end{align}
%`
These are exactly the $w_t(\cdot\mid x)$ functions that agree
with both the known nominal propensities $e_t(x)$
and the MSM in eq.~\eqref{eq: log-odds}. Eq. \eqref{eq: uncertainty-set}
is derived directly from eq.~\eqref{eq: log-odds} by
simple algebraic manipulation.
We define the population CATE bounds under the MSM correspondingly. 
\begin{definition}[CATE Identified Set Under MSM]
	The population bounds under the MSM with parameter $\Gamma$ for the expected potential outcomes are 
	\begin{align}
	\overline{\mu}_t(x; \Gamma) &= \sup_{w_t \in \mathcal{W}_t(x; \Gamma)} \mu_t(w_t; x), \label{eq: pop-outcome-upper}\\
	\underline{\mu}_t(x; \Gamma) &= \inf_{w_t \in \mathcal{W}_t(x; \Gamma)}\mu_t(w_t; x), \label{eq: pop-outcome-lower}
	\end{align}
	%`
	and the population bounds for CATE are
	\begin{align}\label{eq: bounds}
	\overline{\tau}(x; \Gamma) = \overline{\mu}_1(x; \Gamma)  - \underline{\mu}_{0}(x; \Gamma), \\  \underline{\tau}(x; \Gamma) = \underline{\mu}_1(x; \Gamma)  - \overline{\mu}_{0}(x; \Gamma).
	\end{align}
\end{definition}\if0\forarxiv\vspace{-0.5\baselineskip}\fi
Therefore,
the target function we are interested in learning is
the map from $x$ to identifiable CATE intervals:
$$
\mathcal T(x;\Gamma)=[\underline{\tau}(x; \Gamma),\overline{\tau}(x; \Gamma)].
$$
%`

%`
%`
%`
%`
%`
%`
%`
%`
%`
%`
%`
%`
%`
%`
%`
%`
%`
%`
%`
%`
%`
%`
%`
%`
%`
%`

\subsection{The functional interval estimator}\label{subsection: estimate-bounds}
%`
We next develop our functional interval estimator for 
$\mathcal T(x;\Gamma)$. Toward this end,
we consider the following kernel-regression-based estimator for $\mu_t(x)$ based on the \emph{unknown} weights
$\mathbf{W}_t^* = (W^*_{ti})_{i = 1}^n$ based on the 
complete propensity score,
$W_{ti}^* = 1/e_t(X_i, Y_i(t))$:
\begin{equation}\label{eq: estimator}
\hat\mu_{t}(\mathbf{W}^*_t; x) =  \frac{\sum_{i=1}^n\ind{(T_i =  t)}\mathbf{K}(\frac{X_i-x}{h})W^*_{ti}Y_i}{\sum_{i=1}^n\ind{(T_i =  t)}\mathbf{K}(\frac{X_i-x}{h})W^*_{ti}},
\end{equation}
%`
where
$\mathbf{K}(\frac{X_i-x}{h}) = K(\frac{X_{i, 1} - x_1}{h}) \times \dots \times K(\frac{X_{i, d} - x_d}{h})$, $K(\cdot)$ is a univariate kernel function, and $h>0$ is a bandwidth. 
In particular, all we require
of $K$ is that it is bounded and $\int uK(u)du=0,\int u^2K(u)du<\infty$
(see Thm.~\ref{thm: consistent-cate}).
For example, we can use the 
Gaussian kernel $K(u) = \exp(u^2/2)$ or
uniform kernel $K(u) = \ind(\abs{u} \le \frac{1}{2})$.  
\emph{If} we knew the true weights $\mathbf{W}^*_t$, then basic 
results on non-parametric 
regression and inverse-probability weighting would immediately
give that $\hat{\mu}_t(\mathbf{W}^*_t; x) \to \mu_t(x)$
as $n \to \infty$ if $h \to 0$ and $nh^d \to \infty$ 
\citep{pagan1999nonparametric}, \ie, the estimator, eq.~\eqref{eq: estimator}, would be consistent when the complete propensity scores are known.

However, the estimator in eq.~\eqref{eq: estimator} is an infeasible one in practice because $1/W_{ti}^* = e_t(X_i, Y_i(t))$ is unknown and cannot be estimated from any amount of observed data. Instead, we bracket the range of feasible weights and consider how large or small eq.~\eqref{eq: estimator} might be. For $t \in \{0, 1\}$ and $\Gamma\geq1$, we define
$$
\widehat{\mathcal W}_t=\fbraces{\mathbf W_t: \alpha_t(X_i; \Gamma)\leq W_{ti}\leq \beta_t(X_i; \Gamma)~\forall i}\subseteq\R n,
$$
where $\alpha_t(\cdot)$ and $\beta_t(\cdot)$ are defined in \eqref{eq: uncertainty-set}.
Our interval CATE estimator is
\begin{align}
&\widehat{\mathcal T}(x;\Gamma)=
[
\hat{\underline{\tau}}(x)
,~
\hat{\overline{\tau}}(x)
],~~~\text{where}\\
\label{eq: est-cate-bounds}
&\hat{\overline{\tau}}(x) = \hat{\overline{\mu}}_1(x)  - \hat{\underline{\mu}}_{0}(x) ,  ~~ \hat{\underline{\tau}}(x) = \hat{\underline{\mu}}_1(x)  - \hat{\overline{\mu}}_{0}(x), \\
	&\hat{\overline{\mu}}_t(x; \Gamma) = \sup\limits_{\mathbf W_t\in\widehat{\mathcal W}_t}\hat{\mu}_t(\mathbf{W}; x),  \label{eq: est-upper-outcome} \\ 
	&\hat{\underline{\mu}}_t(x; \Gamma) = \inf\limits_{\mathbf W_t\in\widehat{\mathcal W}_t}\hat{\mu}_t(\mathbf{W}; x).   \label{eq: est-lower-outcome} 
\end{align} 
%`
Note that $\alpha_t(x),\beta_t(x)$ depend on $e_t(x)$.
Since we mainly focus on dealing with unobserved confounding, we assume that we know the nominal propensity scores $e_t(x)$ for simplicity as it is in fact identifiable. In Subsection \ref{subsection: practical}, we discuss the estimation of the nominal propensity score in finite samples and the interpretation of the marginal sensitivity model when the propensity score is misspecified. 

\subsection{Computing the interval estimator}

Our interval estimator is defined as an optimization problem over
$n$ weight variables. 
We can simplify this problem by characterizing its solution.
Using optimization duality,
Lemma \ref{lemma: sample-est-formulation} in appendix
shows that, in the solution, each weight variable realizes its bounds
(upper or lower) and that weights are monotone when sorted in
increasing $Y_i$ value. This means that one need only search
for the inflection point.
%`
%`
%`
As summarized in the following proposition, this means that 
the solution is given by a simple discrete line search 
to optimize a unimodal function, after sorting. 
%`

\begin{proposition}\label{prop: sample-est-computation}
	Suppose that we reorder the data so that $Y_1 \le Y_2 \le \dots \le Y_n$. Define the following terms for $k \in \{1, \dots, n\}$, $x \in \mathcal{X}$, and $\Gamma\geq1$:
	\begin{align*}
	\overline{\lambda}(k; x, \Gamma) &= \frac{\sum_{i \le k}\tilde{\alpha}^K_i(t, x; \Gamma)Y_i +  \sum_{i \ge k+1}\tilde{\beta}^K_i(t, x; \Gamma)Y_i}{\sum_{i \le k}\tilde{\alpha}^K_i(t, x; \Gamma)+  \sum_{i \ge k+1}\tilde{\beta}^K_i(t, x; \Gamma)}, \\
	\underline{\lambda}(k; x, \Gamma)  &= \frac{\sum_{i \le k}\tilde{\beta}^K_i(t, x; \Gamma)Y_i +  \sum_{i \ge k+1}\tilde{\alpha}^K_i(t, x; \Gamma) Y_i}{\sum_{i \le k}\tilde{\beta}^K_i(t, x; \Gamma)+  \sum_{i \ge k+1}\tilde{\alpha}^K_i(t, x; \Gamma)},
	\end{align*}
	where 
	\begin{align*}
	\tilde{\alpha}^K_i(t, x; \Gamma) = \mathbb{I}[T_i = t]\alpha_t(X_i; \Gamma)\mathbf{K}(\frac{X_i - x}{h}), \\
	\tilde{\beta}^K_i(t, x; \Gamma) = \mathbb{I}[T_i = t]\beta_t(X_i; \Gamma)\mathbf{K}(\frac{X_i - x}{h}).
	\end{align*}
	Then we have that 
	\begin{align*}
	&\hat{\overline{\mu}}_t(x) = \overline{\lambda}({k}^H(x, \Gamma); x, \Gamma),~
	\hat{\underline{\mu}}_t(x) = \underline{\lambda}({k}^L(x, \Gamma); x, \Gamma),
	\end{align*}
	where
	\begin{align*}
	{k}^H(x, \Gamma) &= \inf\{k=1,\dots,n: \overline\lambda(k;x,  \Gamma) \ge \overline\lambda(k+1;x,  \Gamma)\},\\
	{k}^L(x, \Gamma) &= \inf\{k=1,\dots,n: \underline\lambda(k; x, \Gamma) \le \underline\lambda(k+1; x, \Gamma)\}.
	\end{align*}
%`
%`
%`
%`
%`
%`
%`
%`
%`
%`
%`
%`
%`
%`
%`
%`
%`
%`
%`
%`
%`
%`
\end{proposition}

%`
%`
%`
%`

\subsection{Sharpness guarantees}

We next establish that our interval estimator is \emph{sharp},
\ie, it converges to the identifiable set of CATE values. 
That is to say, as a robust estimator that accounts for possible
confounding, our interval is neither too wide nor too narrow --
asymptotically, it matches exactly what can be hoped to be learned
from any amount of observational data.
The result is based on a new \emph{uniform} 
convergence result that we prove for
the weight-parametrized kernel regression estimator, $\hat\mu_t(W;x)$,
to the weight-parametrized estimand $\mu_t(w_t;x)$.
Although the uniform convergence may in fact not hold in general, 
it holds when restricting to monotone weights, which is where
we leverage our characterization of the optimal solution
to eqs.~\eqref{eq: est-lower-outcome} and \eqref{eq: est-upper-outcome}
as well as
a similar result characterizing the population version in
eqs.~\eqref{eq: pop-outcome-upper} and \eqref{eq: pop-outcome-lower}
using semi-infinite optimization duality \citep{shapiro01}.
%`
%`

\begin{theorem}\label{thm: consistent-cate}
	Suppose that
	\begin{enumerate}[label=\roman*.,align=left,leftmargin=*,labelindent=0in,topsep=0ex,itemsep=0ex,partopsep=0ex,parsep=0ex]
	\item $K$ is bounded, $\int K(u)<\infty$, $\int uK(u)du = 0$, and $\int u^2K(u)du < \infty$,
	\item $n \to \infty$, $h \to 0$ and $nh^{2d} \to \infty$,
	\item $Y$ is a bounded random variable,
	\item $e_t(x)$ and $f_t(y \mid x)$ are twice continuously differentiable with respect to $x$ for any fixed $y\in\mathcal Y$ with bounded first and second derivatives,\footnote{Note that we can also use $\alpha_t(x; \Gamma)$ and $\beta_t(x; \Gamma)$ as the bounds on $W_{ti}$ in the definition of our estimators in eqs.~\eqref{eq: est-upper-outcome} and \eqref{eq: est-lower-outcome}. In this case, we don't need derivative assumptions on $e_t(\cdot)$. In practice, using $x$ or $X_i$ leads to similar results.}
	\item $e_t(x, y)$ is bounded away from 0 and 1 uniformly over 
	$x \in \mathcal{X}$, $y \in \mathcal{Y}$, $t \in \{0, 1\}$.
	\end{enumerate}\if0\forarxiv\vspace{-0.5\baselineskip}\fi
	Then, for $t \in \{0, 1\}$, 
	\begin{align*}
		\hat{\overline{\mu}}_t(x) &\overset{\text{p}}{\to} {\overline{\mu}}_t(x),   \quad \hat{\underline{\mu}}_t(x) \overset{\text{p}}{\to} {\underline{\mu}}_t(x), \\ 
		\hat{\overline{\tau}}(x) &\overset{\text{p}}{\to} {\overline{\tau}}(x),    \quad \hat{\underline{\tau}}(x) \overset{\text{p}}{\to} {\underline{\tau}}(x). 
	\end{align*}
\end{theorem}\if0\forarxiv\vspace{-0.75\baselineskip}\fi
In words,
Theorem \ref{thm: consistent-cate} states that under fairly general assumptions, if the bandwidth $h$ is appropriately chosen, our bounds for both conditional average outcomes and CATE are pointwise consistent and hence sharp.

%`
\subsection{Personalized Decisions from Interval Estimates and Minimax Regret Guarantees} \label{subsection: decision-making}
%`

We next consider how our interval CATE estimate 
can be used for personalized treatment
decisions and prove that the resulting decisions 
rules are asymptotically minimax
optimal. Let us assume that the outcomes $Y_i$ correspond to 
\emph{losses} so that \emph{lower outcomes are better}. Then,
\emph{if} the CATE were known, 
given an individual with covariates
$x$, clearly the optimal treatment decision is 
$t=1$ if 
$\tau(x)<0$ and $t=0$ if $\tau(x)>0$
(and either if $\tau(x)=0$). In other words,
$\pi(x)=\ind(\tau(x)<0)$ minimizes the risk
\begin{align*}
V(\pi;\tau) &= \expect[\pi(X)Y(1) + (1 - \pi(X))Y(0)]
\\&=\E[Y(1)]+\E[\pi(X)\tau(X)]
\end{align*}
over $\pi:\mathcal X\to\{0,1\}$.
If $\tau(x)$ can be point-estimated,
an obvious approach to making personalized decisions is
to threshold an estimator of it. If the estimator is consistent,
%`
this will lead to asymptotically optimal risk.

This, however, breaks down when CATE is unidentifiable and
we only have an interval estimate.
It is not immediately clear how one should threshold an interval.
We next discuss how an approach that thresholds when possible
and otherwise falls back to defaults is minimax optimal.
When CATE is not a single function, there is also no single identifiable value of $V(\pi;\tau)$.
Instead, we focus on the 
worst case regret given by the MSM relative to a default $\pi_0(x)$:
\begin{align*}\textstyle
\overline R_{\pi_0}(\pi;\Gamma) = 
\sup_{\tau(x)\in\mathcal T(x;\Gamma)~\forall x\in\mathcal X}(V(\pi;\tau)  - V(\pi_0;\tau))
%`
%`
%`
\end{align*}
The default represents the decision that would have been taken
in the absence of any of our observational data.
For example, in the medical domain, if there is not enough clinical trial evidence to support treatment approval, then the default may be to not treat, $\pi_0(x) = 0$.
At the population level, the uniformly best 
possible policy we can hope for
is the minimax regret policy:
\begin{equation}\label{eq: minimaxpolicy}\textstyle
\pi^*(\;\cdot\;;\Gamma) \in \argmin_{\pi: \mathcal{X} \to \{0, 1\}} \overline R_{\pi_0}(\pi;\Gamma)
\end{equation}
%`
%`
%`
%`
%`
%`
%`
\begin{proposition}\label{prop: population-opt-policy}
The following is a solution to eq.~\eqref{eq: minimaxpolicy}:
	%`
	\begin{align}\notag
	%`
		\pi^*(x;\Gamma) = 
		\ind(\overline{\tau}(x;\Gamma) \le 0) +  \pi_0(x)\ind(
		\underline{\tau}(x;\Gamma) < 0 < \overline{\tau}(x;\Gamma))
		%`
	\end{align}
\end{proposition}\if0\forarxiv\vspace{-\baselineskip}\fi
This minimax-optimal policy always treats when $\overline{\tau}(x;\Gamma) \leq 0$ and never treats when $\underline{\tau}(x;\Gamma) \geq 0$ because in those cases the best choice is unambiguous. 
Whenever the bounds contain 0, 
the best we could hope for is 0 regret,
which we can always achieve by mimicking $\pi_0$.

%`

%`
%`
%`
%`
%`

%`
%`
%`
%`
%`
%`
%`

Next, we prove that if we approximate the true minimax-optimal policy,
by plugging in our own interval CATE estimates 
in place of the population estimands, 
then we will achieve optimal minimax regret asymptotically.
%`
\begin{theorem}\label{thm: policy}
Let
$$
\hat\pi(x;\Gamma) = 
		\ind(\hat{\overline{\tau}}(x;\Gamma) \le 0) +  \pi_0(x)\ind(
		\hat{\underline{\tau}}(x;\Gamma) < 0 \leq \hat{\overline{\tau}}(x;\Gamma)).
$$
	%`
	Then, under the assumptions of Theorem \ref{thm: consistent-cate},
	\[
		%`
		\overline R_{\pi_0}(\hat\pi(\cdot;\Gamma);\Gamma)
		\overset{\text{p}}{\to}
		\min_{\pi: \mathcal{X} \to \{0, 1\}} \overline R_{\pi_0}(\pi;\Gamma)=
		\overline R_{\pi_0}(\pi^*(\cdot;\Gamma);\Gamma)
	\]
	%`
\end{theorem}

%`

\subsection{Extension: interval estimates for the partial conditional average treatment effect}\label{subsection: partial-cate}
In subsections \ref{subsection: population-bounds} - \ref{subsection: decision-making}, we consider CATE conditioned on all observed confounders $X$. However, in many applications, we may be interested in heterogeneity of treatment effect in only a few variables, conditioning on only a subset of variables $X_S$, with $S \subset \{1, \dots, n\}$ as the corresponding index set. For example, in medical applications, fewer rather than more variables are preferred in a personalized decision rule due to cost and interpretability considerations \cite{qian2011performance}. 
In other cases, only a subset of covariates are available at
test time to use as inputs for an effect prediction.
Therefore, we consider estimation of the \textit{partial} conditional average treatment effect (PCATE): 
%`
\begin{equation}\label{eq: pcate}
\tau(x_S)= \mu_1(x_S) -  \mu_{0}(x_S),
\end{equation}
where $\mu_t(x_S) = \Efb{Y(t) \mid X_S = x_S}$ for $t \in \{0, 1\}$. 

Analogously, we define $f_t(y, x_{S^c} \mid x_S)$ as the joint conditional density function of $(T = t, Y(t), X_{S^c})$ given $X_S = x_S$, where $S^c = \{1, \dots, n\} \setminus S$ denotes the complement of $S$. We further define the following for a weight functional $w_t^P(\cdot\ , \cdot \mid x_S)$: 
\begin{equation}\label{eq: partial-outcome}
\mu_t(w_t^P; x_S) = \frac{\iint yw_t^P( x_{S^c}, y \mid x_S)f_t(y, x_{S^c} \mid x_S)dydx_{S^c}}{\iint w_t^P(x_{S^c}, y \mid x_S)f_t(y, x_{S^c} \mid x_S)dydx_{S^c}}
\end{equation}
Note that $\mu_t(w_t^P; x_S) =\mu_t(x_S)$ when  $w_t^P(x_{S^c}, y \mid x_S) = \frac{1}{e_t(x_S, x_{S^c}, y)}$. We therefore define the population interval estimands for PCATE under the MSM as follows. 

\begin{definition}[PCATE Identified Set Under MSM]
	The population bounds under the MSM with parameter $\Gamma$
	for the partial expected potential outcomes and PCATE are 
	\begin{align}
	\overline{\mu}_t(x_S; \Gamma) &= \sup_{w_t^P \in \mathcal{W}^P_t(x_S; \Gamma)} \mu(w_t^P; x_S), \label{eq: partial-outcome-upper}\\
	\underline{\mu}_t(x_S; \Gamma) &=\inf_{w_t^P \in \mathcal{W}^P_t(x_S; \Gamma)} \mu(w_t^P; x_S), \label{eq: partial-outcome-lower}\\
	\overline{\tau}(x_S; \Gamma) &= \overline{\mu}_1(x_S; \Gamma)  - \underline{\mu}_{0}(x_S; \Gamma), \\  \underline{\tau}(x_S; \Gamma) &= \underline{\mu}_1(x_S; \Gamma)  - \overline{\mu}_{0}(x_S; \Gamma).
	\end{align} 
	where  
	$\mu(w_t^P; x_S)$ is defined in \eqref{eq: partial-outcome}, and 
	\begin{align*}
	&\mathcal{W}^P_t(x_S; \Gamma) =\bigg \{
	w_t^P: w_t^P(x_{S^c}, y \mid x_S) \in  [\alpha_t(x_S, x_{S^c}; \Gamma),\\
	& \qquad\qquad\qquad\qquad \beta_t(x_S, x_{S^c}; \Gamma)],  
	\forall x_{S^c} \in \mathcal{X}_{S^c}, \forall y \in \mathcal{Y} 
	\bigg\}. 
	\end{align*}
	%`
	%`
%`
%`
%`
%`
\end{definition}\if0\forarxiv\vspace{-\baselineskip}\fi
We can extend our interval estimators to the above: 
\begin{align}
\hat{\mu}_{t}(\mathbf{W}_t; x_S) &=  \frac{\sum_{i=1}^n\ind{(T_i =  t)}\mathbf{K}(\frac{X_{i, S}-x_S}{h})W_{ti}Y_i}{\sum_{i=1}^n\ind{(T_i =  t)}\mathbf{K}(\frac{X_{i, S}-x_S}{h})W_{ti}}, \label{est-partial-outcome}\\
\hat{\overline{\mu}}_t(x_S; \Gamma)&= \sup\limits_{\mathbf W_t\in\widehat{\mathcal W}_t}\hat{\mu}_t(\mathbf{W}_t; x_S),  \label{eq: est-partial-upper-outcome} \\ 
\hat{\underline{\mu}}_t(x_S; \Gamma) &= \inf\limits_{\mathbf W_t\in\widehat{\mathcal W}_t}\hat{\mu}_t(\mathbf{W}_t; x_S)   \label{eq: est-partial-lower-outcome}  
%`
%`
\end{align} 
These PCATE interval estimators use the partial covariates $X_S$ in the kernel function but the complete observed covariates $X$ in the nominal propensity score (in $\alpha_t(\cdot)$ and $\beta_t(\cdot)$),\footnote{We could also use $\alpha_t(x_S, X_{i, S^c}; \Gamma)$ and $\beta_t(x_S, X_{i, S^c}; \Gamma)$ here. See also footnote 1.} compared to CATE interval estimators, eqs. \eqref{eq: estimator} and \eqref{eq: est-cate-bounds}, that use $X$ in both the kernel function and nominal propensity score. In appendix section \ref{appendix: PCAT}, we prove appropriate analogues of Theorems \ref{thm: consistent-cate} and \ref{thm: policy} for our PCATE interval estimators under analogous assumptions. In this way, we can use the complete observed covariates $X$ in the nominal propensity scores to adjust for confounding as much as possible, so that unobserved confounding is minimal, 
while only estimating heterogeneity in a subset of interesting 
covariates.
%`

\subsection{Practical Considerations}\label{subsection: practical}
%`
%`
%`
%`
%`
%`
%`
%`
%`
%`
%`
%`
%`
%`
%`
%`
%`

%`
%`
%`
%`
%`
%`
%`
 
%`
%`
%`
%`

\textbf{Boundary bias:} If the space $\mathcal{X}$ is bounded, then kernel-regression-based estimators may have high bias at points $x$ near the boundaries. This can be alleviated by truncating kernels at the boundary and corrected by replacing any kernel
term of the form $\mathbf K(\frac{X_i-x}h)$ by a re-normalized version
$\mathbf K(\frac{X_i-x}h)/
\int_{x'\in\mathcal X}\mathbf K(\frac{X_i-x'}h)dx'$
so that all kernel terms have the same integral over the bounded 
$\mathcal X$ \cite{gasser1979kernel,kheireddine2016boundary}.
We take this approach in our experiments.

\textbf{Propensity score estimation:} Although our theoretical results in section \ref{sec: method} assume that the nominal propensity score $e_t(x)$ is known, these results still hold if we use a consistent estimator for it. Indeed, $e_t(x)$ \emph{is} identifiable. Recently, a variety of nonparametric machine learning methods were proposed to estimate propensity score reliably \citep{mccaffrey2004propensity, lee2010improving}. These, for example, may be used. When parametric estimators are used for propensity score estimation, \eg, linear logistic regression, model misspecification error may occur. In this case, we can interpret the marginal sensitivity model as the log odds ratio bound between the complete propensity score and the best parametric approximation of the nominal propensity score. Consequently, the resulting CATE sensitivity bounds also incorporate model misspecifcation uncertainty. See \cite{zhaosmall17} for more details on marginal sensitivity model for parametric propensity score. 

\paragraph{Selection of the sensitivity parameter $\Gamma$.} The parameter $\Gamma$ bounds the magnitude of the effects of unobserved confounders on selection, which is usually unknown.  \cite{hsu2013calibrating} suggest calibrating the assumed effect of the unobserved confounders to the effect of observed covariates. For example, we can compute the effect of omitting each observed covariate on the log odds ratio of the propensity score and use domain knowledge to assess plausible ranges of $\Gamma$ to determine if we could have omitted a variable that could have as large an effect as the observed one.
%`
%`
\section{Experiments}\if0\forarxiv\vspace{-5pt}\fi
\paragraph{Simulated Data.}
We first consider an one-dimensional example illustrating the effects of unobserved confounding on conditional average treatment effect estimation. We generate a binary unobserved confounder $u \sim \op{Bern}(\nicefrac{1}{2})$ (independent of all else), and covariate $X \sim \op{Unif}[-2,2]$. We fix the nominal propensity score as
$e(x) = \sigma(0.75 x +0.5)$. For the sake of demonstration, we
fix an underlying ``true'' $\Gamma^*$ value and
set the complete propensity scores as 
$e(x,u) = \frac{u}{\alpha_t(x; \Gamma^*)} + \frac{1-u}{\beta_t(x; \Gamma^*)}$ and sample $T \sim \op{Bern}(e(x,u))$.
This makes the complete propensities achieve the extremal MSM bounds corresponding to $\Gamma^*$, with $u$ controlling which bound we reach.
%`
%`
%`
\if1\forarxiv
\begin{figure}[t!]%
\begin{minipage}[t]{0.475\textwidth}%
\includegraphics[width=\textwidth]{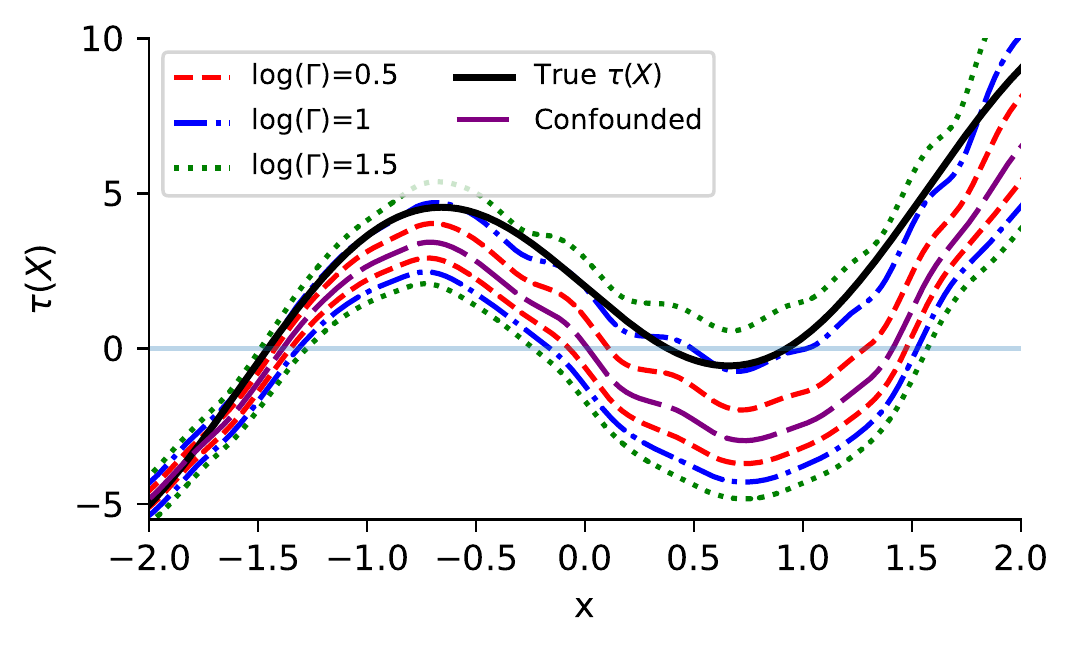}%
\caption{Bounds on CATE for differing values of $\Gamma$ (dashed), compared to original confounded kernel regression (purple dashed) and true CATE (black). }\label{fig-1d}%
\end{minipage}%
\hspace{0.045\textwidth}%
\begin{minipage}[t]{0.475\textwidth}%
\includegraphics[width=\textwidth]{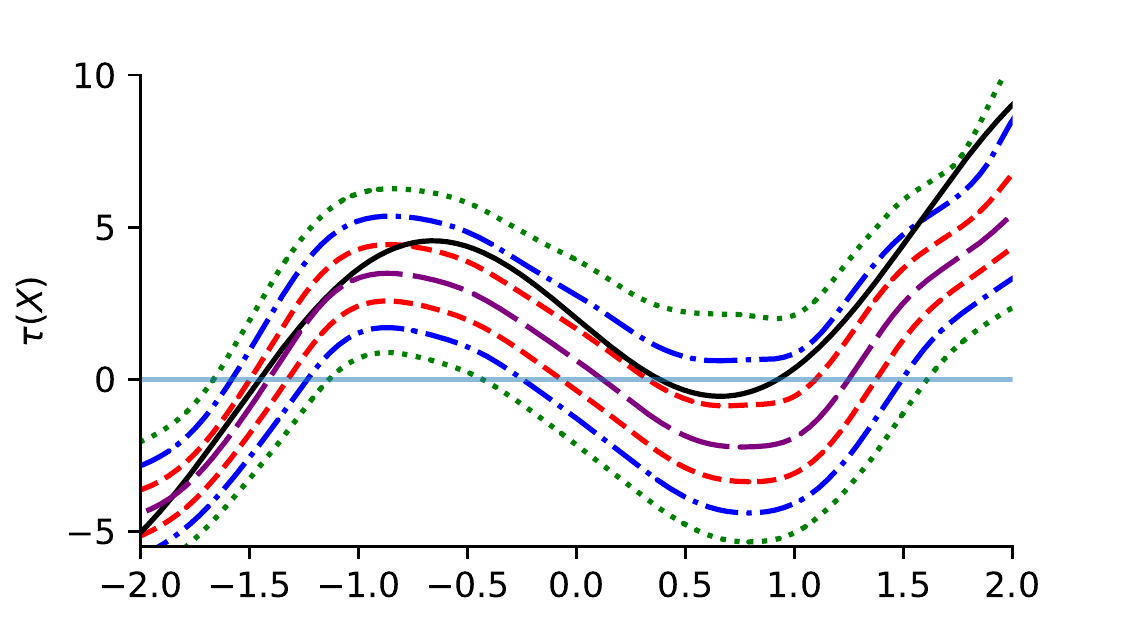}%
\caption{Bounds on PCATE. Legend same as in Fig~\ref{fig-1d}.}
\label{fig:partialcate-1d-ablation-nobeta-x}%
\end{minipage}%
\end{figure}
\fi
\if0\forarxiv
\begin{figure}[t!]%
\includegraphics[width=0.5\textwidth]{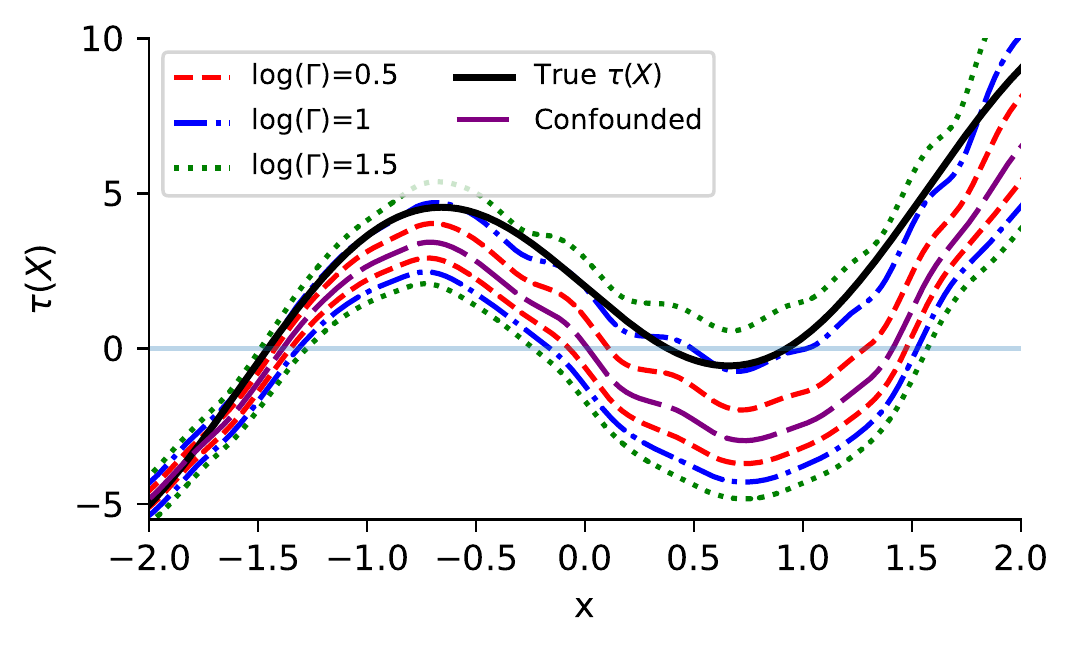}%
\caption{Bounds on CATE for differing values of $\Gamma$ (dashed), compared to original confounded kernel regression (purple dashed) and true CATE (black). }\label{fig-1d}%
\includegraphics[width=0.5\textwidth]{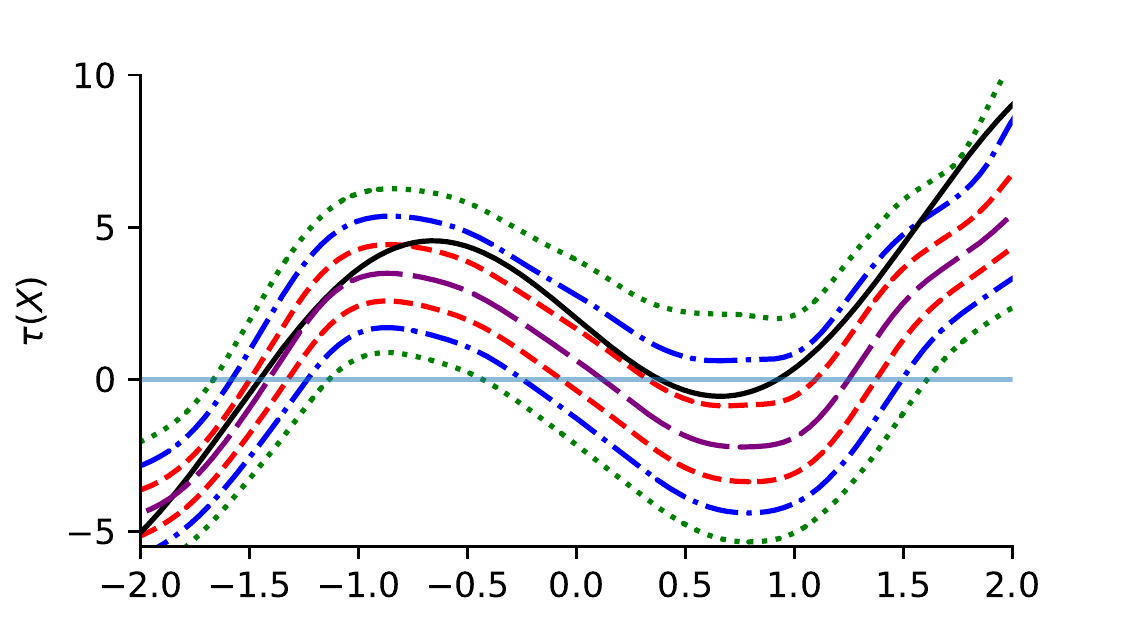}%
\caption{Bounds on PCATE. Legend same as in Fig~\ref{fig-1d}.}
\label{fig:partialcate-1d-ablation-nobeta-x}%
\end{figure}
\fi
%`
%`
%`
We choose an outcome model to yield a nonlinear CATE, with linear confounding terms and noise randomly generated as $\epsilon \sim N(0,1)$: 
%`
%`
\begin{align*}Y(t) = &(2t-1) X + (2t-1)-  2\sin(2(2t-1) X)\\& - 2 (2u-1)(1 + 0.5X) + \epsilon\end{align*}
When we learn the confounded effect estimate, $\tilde{\tau}(X)$, from data as in eq.~\eqref{eq: fake-cate}, we  incur a confounding term ($\tilde{\tau}(x) - \tau(x)$) that grows in magnitude with positive $x$: 
%`
%`
%`
%`
\begin{align}
2(2+x)( \Pr[u=1 \mid \substack{X=x\\T=1}] - \Pr[u=1\mid \substack{X=x\\T=0}] ).  \label{eq: confound-term}
\end{align}
%`
 In Fig.~\ref{fig-1d}, we compute the bounds using our estimators, eqs.~\eqref{eq: est-upper-outcome} and \eqref{eq: est-lower-outcome}, 
 for varying choices of $\Gamma$
 %`
 on a dataset with $n=2000$ where $\log\Gamma^*=1$. We use a Gaussian kernel with bandwidths chosen by leave-one-out cross-validation for the task of unweighted regression in each treatment arm. The bounds are centered at the confounded kernel regression estimate of CATE (purple long-dashed line). 
%`

 By construction, the deviation of the confounded CATE from the true CATE, eq. \eqref{eq: confound-term}, is greater for more positive $x$. Correspondingly, as can be seen in the figure, our approach learns bounds whose widths reflect the appropriate ``size'' of confounding at each $x$. While the confounded estimation suggests a large region, $x \in [0, 1.25]$, where treatment $\pi(x)=1$ is optimal, the true CATE suggests that treatment at many of these $x$ values is harmful. Correspondingly, our interval bounds for $\log(\Gamma)$ correctly specified as $\geq1$ indeed suggests that the benefit of treatment in this region is ambiguous and treatment may be harmful. 

\begin{table*}[t!]
	\centering 
	\caption{Policy risk for various policies under data generating processes with different $\Gamma^*$ (lower is better)} \label{tbl-syn}
	%`
	\setlength{\tabcolsep}{\if1\forarxiv1.25\fi\if0\forarxiv2\fi pt}
	\begin{tabular}{lp{0cm}lllp{0cm}lllp{0cm}lp{0cm}l}\toprule
		\if1\forarxiv\multirow{2}{*}{\rotatebox[origin=c]{90}{$\log(\Gamma^*)$}}\fi&&\multicolumn{3}{c}{$n=1000$}
		&&\multicolumn{3}{c}{$n=5000$}&&
		\multicolumn{1}{c}{$n=1000$}&&\multicolumn{1}{c}{$n=\infty$}\\\cmidrule{3-5}\cmidrule{7-9}\cmidrule{11-11}\cmidrule{13-13}
		 \if0\forarxiv$\log(\Gamma^*)$\fi&& \multicolumn{1}{c}{$\hat\pi(x;e^{0.5})$} & \multicolumn{1}{c}{$\hat\pi(x;e^{1})$} & \multicolumn{1}{c}{$\hat\pi(x;e^{1.5})$}  &&  \multicolumn{1}{c}{$\hat\pi(x;e^{0.5})$} & \multicolumn{1}{c}{$\hat\pi(x;e^{1})$} & \multicolumn{1}{c}{$\hat\pi(x;e^{1.5})$}  && \multicolumn{1}{c}{$\ind(\hat{\tilde\tau}(x)<0)$} && \multicolumn{1}{c}{$\ind({\tau(x)}<0)$}   \\
		\midrule
		0.5&&$\mathbf{-1.65}\pm 0.01$  & $-1.64\pm 0.01$  & $-1.63\pm 0.00$ &&$\mathbf{-1.68}$ & $-1.63$ & $-1.63$&& $-1.60 \pm 0.00$ && 	$-1.68$\\
		1&&$-1.58\pm 0.02$  & $\mathbf{-1.64}\pm 0.02$  & $-1.64\pm 0.01$ &&$-1.62$ & $\mathbf{-1.67}$ & $-1.63$ && $-1.48\pm 0.04$ &&$-1.68$ \\
		1.5&&$-1.51\pm 0.02$  & $-1.60\pm 0.02$  & $\mathbf{-1.63}\pm 0.02$&&$-1.52$ & $-1.63$ & $\mathbf{-1.67}$&&$-1.36\pm 0.04$ && $-1.68$ \\\bottomrule
	\end{tabular}
\end{table*}

%`
%`
%`
%`
%`
%`
%`
%`
%`
%`
%`
%`
%`

%`
%`
%`
%`
%`
%`
%`
%`
%`
%`
%`
%`
%`
%`
%`
%`
%`
%`
%`
%`
%`
%`
%`
%`
%`
%`
%`
%`
%`
%`
%`
%`
%`
%`
%`

In Table~\ref{tbl-syn}, we compare the \emph{true} 
policy values, $V(\pi;\tau)$, achieved by the decision rules derived from our interval CATE estimates, following Section \ref{subsection: decision-making} and letting $\pi_0(x)=0$ (never treat).
%`
We consider 20 Monte Carlo replications for each setting of $\Gamma^*$ and report $95\%$ confidence intervals.
Any omitted confidence interval is smaller than $\pm0.01$.
Note that on the diagonal of Table~\ref{tbl-syn}, we assess a policy with a ``well-specified'' $\Gamma$ equal to $\Gamma^*$, which achieves the best risk for the corresponding data generating process.
The case of $n=5000$ essentially gives the population-level
optimal minimax regret.
Finally, for comparison, we include the policy values of
both the thresholding policy based 
on the confounded CATE estimated learned
by IPW-weighted kernel regression using nominal propensities 
($\hat{\tilde\tau}(x)$)
and the truly optimal policy based on the true (and unknowable) $\tau$.
%`
The policy value of the confounded policy suffers in comparison to the policies from our estimated bounds. Specifying an overly conservative $\Gamma$ achieves similar risk in this setting, while underspecifying $\Gamma$ compared to the true $\Gamma^*$ incurs greater loss.

We next illustrate the case of learning the PCATE using our interval estimators in eqs.~\eqref{est-partial-outcome} and \eqref{eq: est-partial-lower-outcome}. In Fig.~\ref{fig:partialcate-1d-ablation-nobeta-x}, we show the same CATE specification used in Fig.~\ref{fig-1d}, but introduce additional confounders which impact selection to illustrate the use of this approach with higher-dimensional observed covariates. We consider observed covariates $X \in \mathbb{R}^3$, uniformly generated on $[-1,1]^3$, where heterogeneity in treatment effect is only due to $x_S$, $S=\{1\}$, the first dimension. That is, we specify the outcome model for $t \in \{0,1\}$ as:
%`
\begin{align*}Y(t) = &~(2t-1)X_S +(2t-1) -  2\sin(2(2t-1) X_S) \\&- 2 (2u-1)(1 +  0.5X_S) + \beta_x^\top X + \epsilon \end{align*}
We fix the nominal propensities as $ e(x) = \sigma(\theta^\top x +0.5) $, with $\theta = [0.75, -0.5, 0.5]$ and the outcome coefficient vector $\beta_x = [0.5, 0.5, 0.5]$. Again, we set the propensity scores such that the complete propensities achieve the extremal bounds. % for the generating $\Gamma^*$.
Note that additional confounding dimensions will tend to increase the outcome variation for any given $x_S$ value, so the bounds are wider in Fig.~\ref{fig:partialcate-1d-ablation-nobeta-x} for the same fixed value of $x$ and $\Gamma$, though our approach recovers the appropriate structure on the CATE function.

 \paragraph{Assessment on Real-World Data: Hormone Replacement Therapy.} 
To illustrate the impacts of unobserved confounding, we consider a case study of a parallel clinical trial and large observational study from the Women's Health Initiative \citep{to-whi-03}. Hormone replacement therapy (HRT) was the treatment of interest: previous observational correlations suggested protective effects for onset of chronic (including cardiovascular) disease. While the clinical trial was halted early due to dramatically increased incidence of heart attacks, the observational study evidence actually suggested preventive effects, prompting further study to reconcile these conflicting findings based on unobserved confounding in the observational study \citep{prentice-whi-05,Lawlor04,roussouw-whi-13}. Follow-up studies suggest benefits of HRT for younger women
%`
\citep{bakour2015latest}. 

We consider a simple example of learning an optimal treatment assignment policy based on age to reduce endline systolic blood pressure, which serves as a proxy outcome for protective effects against cardiovascular disease. Thus we consider learning the PCATE for $S=\{\text{age}\}$ while controlling for \emph{all} observed baseline variables. The observed covariates are 30-dimensional (after binary encodings of categorical variables) and include factors such as demographics, smoking habits, cardiovascular health history, and other comorbidities (\eg, diabetes and myocardial infection). We restrict attention to a complete-case subset of the clinical-trial data ($n=14266$), and a subset of the observational study ($n=2657$). 

\if1\forarxiv
\begin{figure}[t!]%
\begin{minipage}[t]{0.475\textwidth}%
\includegraphics[width=\textwidth]{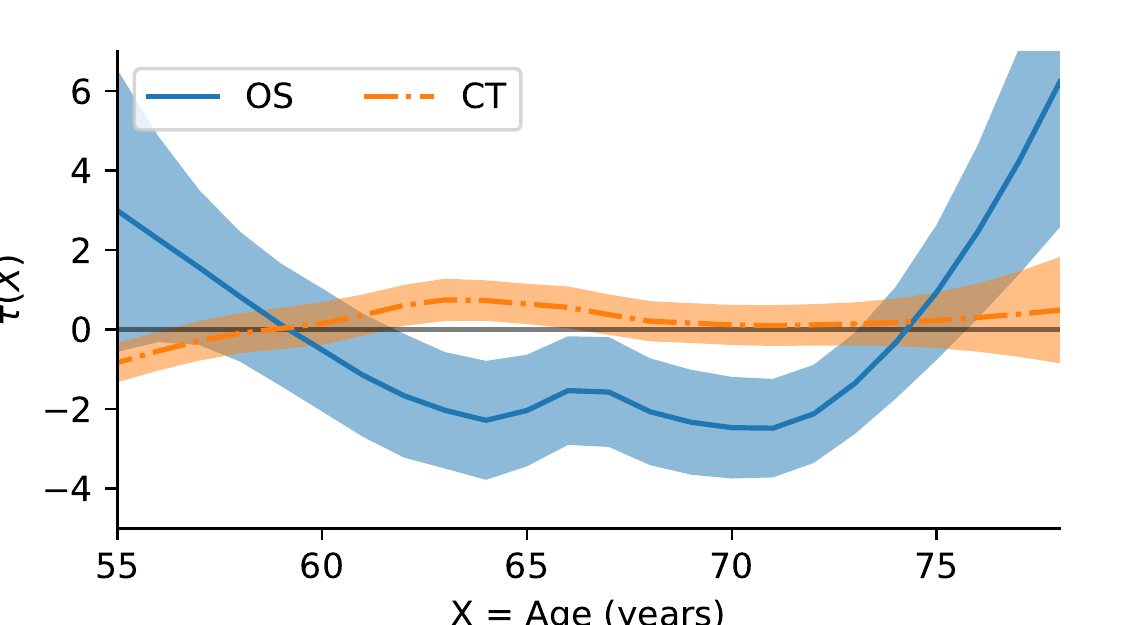}\caption{Comparison of CATE estimated from unconfounded clinical trial data (CT) vs. confounded observational study data (OS) by a difference of LOESS regressions. Confounding leads to the opposite conclusions.}\label{fig-CT-OS-cate}
\end{minipage}%
\hspace{0.045\textwidth}%
\begin{minipage}[t]{0.475\textwidth}%
\includegraphics[width=\textwidth]{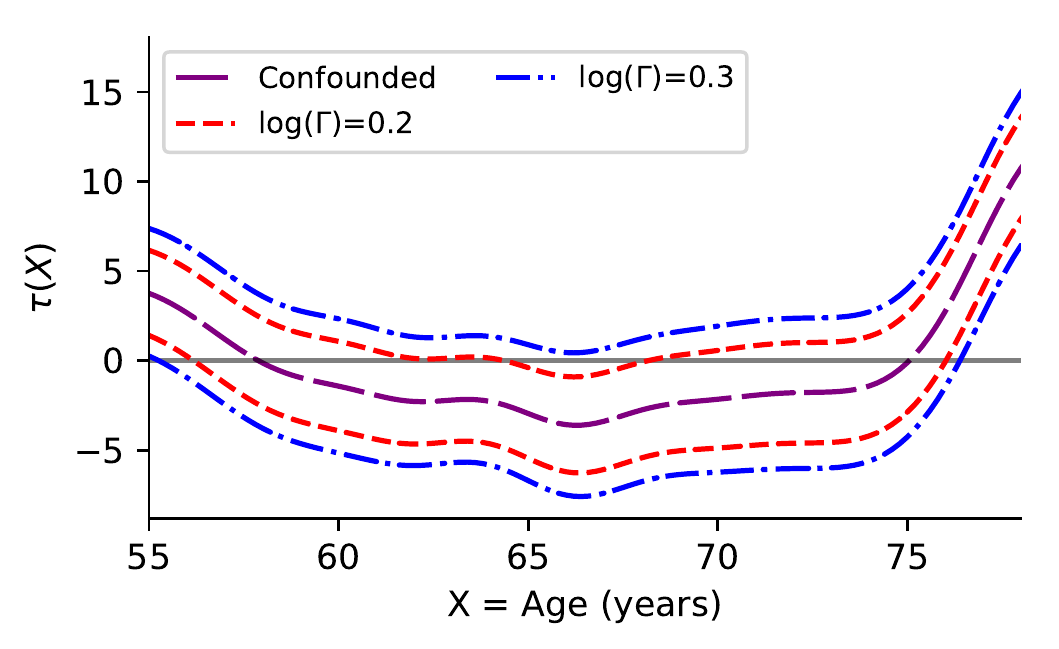}
\caption{Bounds on CATE estimated from the WHI observational study. The observational study is highly sensitive to unobserved confounding.}\label{fig-OS-cate}\end{minipage}%
\end{figure}
\fi
\if0\forarxiv
\begin{figure}[t!]
%`
\includegraphics[width=0.5\textwidth]{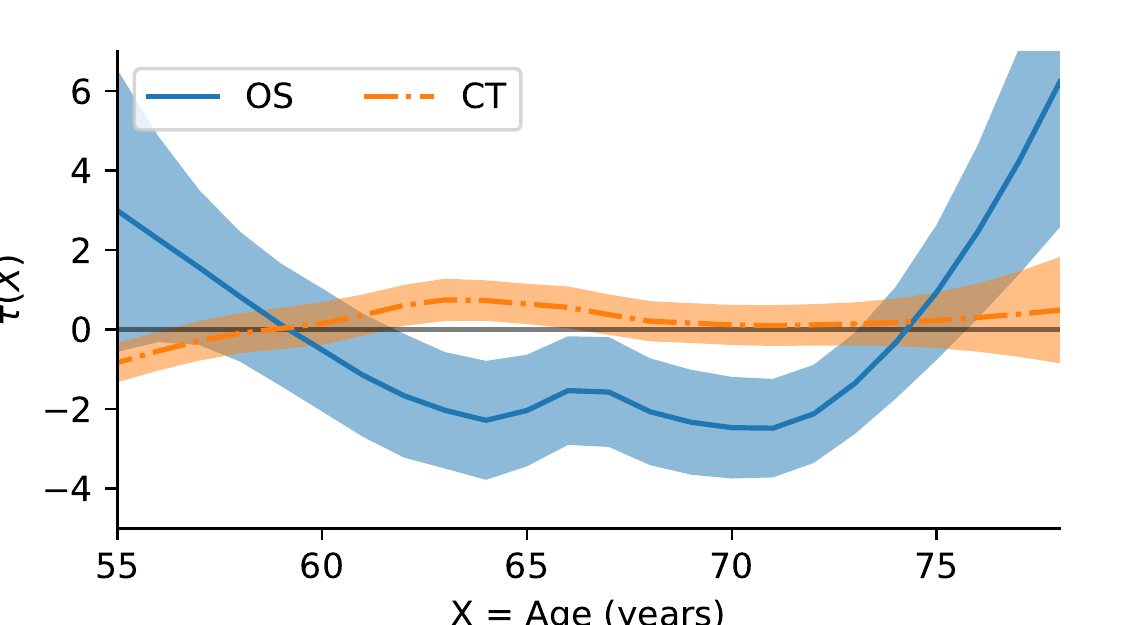}\caption{Comparison of CATE estimated from unconfounded clinical trial data (CT) vs. confounded observational study data (OS) by a difference of LOESS regressions. Confounding leads to the opposite conclusions.}\label{fig-CT-OS-cate}
\includegraphics[width=0.5\textwidth]{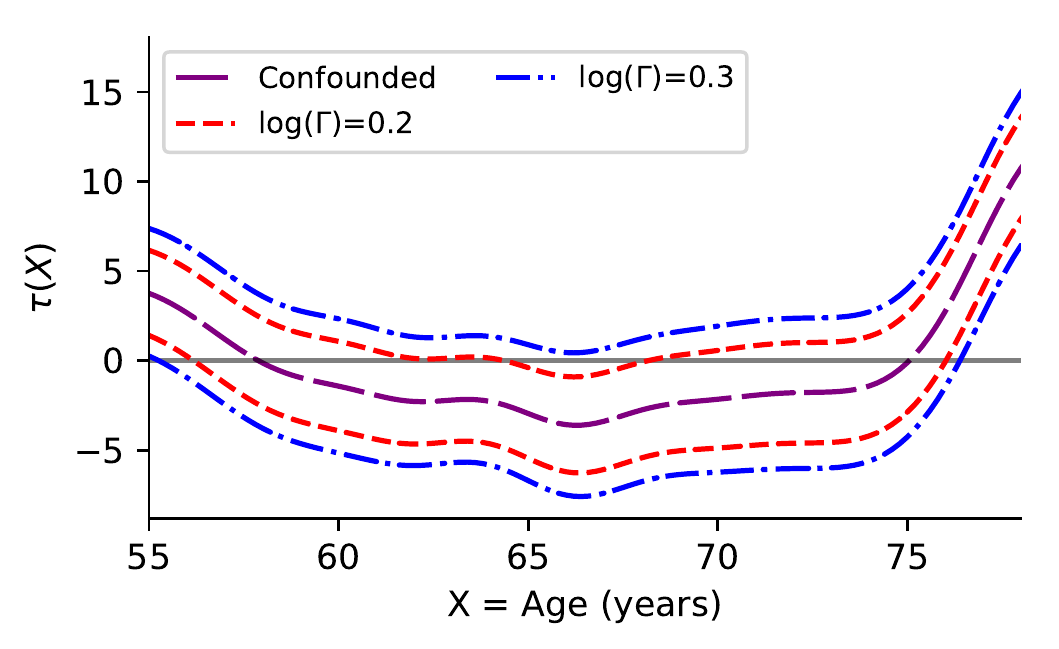}
\caption{Bounds on CATE estimated from the WHI observational study. The observational study is highly sensitive to unobserved confounding.}\label{fig-OS-cate}
\end{figure}
\fi

For comparing findings from the clinical trial and observational study, Fig.~\ref{fig-CT-OS-cate} plots estimates of the partial conditional average treatment effect on systolic blood pressure over age by a difference of LOESS regressions. In the clinical trial (CT, orange) we used a simple regression on age and in the observational study (OS, blue) we used IPW-weighted regression with propensities estimated using all observed baseline variables. A negative CATE suggests that HRT reduces systolic blood pressure and might have protective effects against cardiovascular disease. The clinical trial CATE is not statistically significantly different from zero for ages above 67, though the CATE becomes negative for the youngest women in the study. The observational CATE crucially displays the \textit{opposite} trend, suggesting that treatment is statistically significantly beneficial for women of ages 62-73. We display 90\% confidence intervals obtained from a 95\% confidence interval for individual regressions within each treatment arm. 

In Fig.~\ref{fig-OS-cate}, we apply our method to  estimate bounds on $\tau(X_{\text{age}})$. We estimate propensity scores using logistic regression. 
%`
Our bounds suggest that the estimated CATE from the observational study is highly sensitive to potential unobserved confounding: for sensitivity parameter as low as $\log(\Gamma)=0.2$ ($\Gamma=1.22$), we see that $\tau(x) = 0$ is included in the sensitivity bounds for nearly all individuals such that we would likely prefer to default to less intervention. To interpret this value of $\Gamma$, we compute the distribution of instance-wise $\Gamma_{i,j}$ parameters between the propensity estimated from all covariates for individual $i$, and the propensity estimated under dropping each covariate dimension $j$:
$\Gamma_{i,j} = {\frac{(1 - e_{t_i}(X_i))e_{t_i}(X_{i,-j})}{e_{t_i}(X_i)(1 - e_{t_i}(X_{i,-j}))}}$.
 The maximal such $\Gamma$ value is observed by dropping the indicator for 1-4 cigarettes smoked per day, which leads to a maximal $\Gamma = 1.17$  value.
 \section{Conclusion}
We developed a functional interval estimator that
 provides bounds on individual-level causal effects under 
 realistic violations of unconfoundedness. Our estimators, which we prove are sharp for the tightest bounds possible, use a weighted kernel estimator with weights that vary 
 adversarially over an uncertainty set consistent with a sensitivity model. We study the implications for decision rules, and assess both our bounds and derived decisions on both simulated and real-world data.
\bibliography{sensitivity}

\begin{thebibliography}{40}
\providecommand{\natexlab}[1]{#1}
\providecommand{\url}[1]{\texttt{#1}}
\expandafter\ifx\csname urlstyle\endcsname\relax
  \providecommand{\doi}[1]{doi: #1}\else
  \providecommand{\doi}{doi: \begingroup \urlstyle{rm}\Url}\fi

\bibitem[Abrevaya et~al.(2015)Abrevaya, Hsu, and Lieli]{abrevaya2015estimating}
J.~Abrevaya, Y.-C. Hsu, and R.~P. Lieli.
\newblock Estimating conditional average treatment effects.
\newblock \emph{Journal of Business \& Economic Statistics}, 33\penalty0
  (4):\penalty0 485--505, 2015.

\bibitem[Aronow and Lee(2012)]{aronowlee12}
P.~Aronow and D.~Lee.
\newblock Interval estimation of population means under unknown but bounded
  probabilities of sample selection.
\newblock \emph{Biometrika}, 2012.

\bibitem[Athey and Imbens(2016)]{athey2016recursive}
S.~Athey and G.~Imbens.
\newblock Recursive partitioning for heterogeneous causal effects.
\newblock \emph{Proceedings of the National Academy of Sciences}, 113\penalty0
  (27):\penalty0 7353--7360, 2016.

\bibitem[Athey and Wager(2017)]{athey2017efficient}
S.~Athey and S.~Wager.
\newblock Efficient policy learning.
\newblock \emph{arXiv preprint arXiv:1702.02896}, 2017.

\bibitem[Bakour and Williamson(2015)]{bakour2015latest}
S.~H. Bakour and J.~Williamson.
\newblock Latest evidence on using hormone replacement therapy in the
  menopause.
\newblock \emph{The Obstetrician \& Gynaecologist}, 17\penalty0 (1):\penalty0
  20--28, 2015.

\bibitem[Cornfield et~al.(1959)Cornfield, Haenszel, Hammond, Lilienfeld,
  Shimkin, and Wynder]{cornfield1959smoking}
J.~Cornfield, W.~Haenszel, E.~C. Hammond, A.~M. Lilienfeld, M.~B. Shimkin, and
  E.~L. Wynder.
\newblock Smoking and lung cancer: recent evidence and a discussion of some
  questions.
\newblock \emph{Journal of the National Cancer institute}, 22\penalty0
  (1):\penalty0 173--203, 1959.

\bibitem[Dorie et~al.(2016)Dorie, Harada, Carnegie, and Hill]{doriehchill16}
V.~Dorie, M.~Harada, N.~B. Carnegie, and J.~Hill.
\newblock A flexible, interpretable framework for assessing sensitivity to
  unmeasured confounding.
\newblock \emph{Statistics in Medicine}, 2016.

\bibitem[Dudik et~al.(2014)Dudik, Erhan, Langford, and Li]{dell2014}
M.~Dudik, D.~Erhan, J.~Langford, and L.~Li.
\newblock Doubly robust policy evaluation and optimization.
\newblock \emph{Statistical Science}, 2014.

\bibitem[Gasser and M{\"u}ller(1979)]{gasser1979kernel}
T.~Gasser and H.-G. M{\"u}ller.
\newblock Kernel estimation of regression functions.
\newblock In \emph{Smoothing techniques for curve estimation}, pages 23--68.
  Springer, 1979.

\bibitem[Green and Kern(2010)]{green2010modeling}
D.~P. Green and H.~L. Kern.
\newblock Modeling heterogeneous treatment effects in large-scale experiments
  using bayesian additive regression trees.
\newblock In \emph{The annual summer meeting of the society of political
  methodology}, 2010.

\bibitem[Hartman et~al.(2015)Hartman, Grieve, Ramsahai, and
  Sekhon]{hartman2015sate}
E.~Hartman, R.~Grieve, R.~Ramsahai, and J.~S. Sekhon.
\newblock From sate to patt: combining experimental with observational studies
  to estimate population treatment effects.
\newblock \emph{JR Stat. Soc. Ser. A Stat. Soc.(forthcoming). doi},
  10:\penalty0 1111, 2015.

\bibitem[Hsu and Small(2013)]{hsu2013calibrating}
J.~Y. Hsu and D.~S. Small.
\newblock Calibrating sensitivity analyses to observed covariates in
  observational studies.
\newblock \emph{Biometrics}, 69\penalty0 (4):\penalty0 803--811, 2013.

\bibitem[Johansson et~al.(2018)Johansson, Kallus, Shalit, and
  Sontag]{johansson-k-s-s18}
F.~D. Johansson, N.~Kallus, U.~Shalit, and D.~Sontag.
\newblock Learning weighted representations for generalization across designs.
\newblock \emph{on ArXiv}, 2018.

\bibitem[Kallus(2017)]{kallus2016recursive}
N.~Kallus.
\newblock Recursive partitioning for personalization using observational data.
\newblock In \emph{International Conference on Machine Learning (ICML)}, pages
  1789--1798, 2017.

\bibitem[Kallus(2018)]{kallus2017balanced}
N.~Kallus.
\newblock Balanced policy evaluation and learning.
\newblock \emph{To appear in Proceedings of Neural Information Processing
  Systems}, 2018.

\bibitem[Kallus and Zhou(2018{\natexlab{a}})]{kallus2017off}
N.~Kallus and A.~Zhou.
\newblock Policy evaluation and optimization with continuous treatments.
\newblock In \emph{International Conference on Artificial Intelligence and
  Statistics (AISTATS)}, pages 1243--1251, 2018{\natexlab{a}}.

\bibitem[Kallus and Zhou(2018{\natexlab{b}})]{kz18}
N.~Kallus and A.~Zhou.
\newblock Confounding-robust policy improvement.
\newblock \emph{To appear in Proceedings of Neural Information Processing
  Systems}, 2018{\natexlab{b}}.

\bibitem[Kheireddine(2016)]{kheireddine2016boundary}
S.~Kheireddine.
\newblock \emph{On Boundary Correction in Kernel Estimation}.
\newblock PhD thesis, Universit{\'e} Mohamed Khider-Biskra, 2016.

\bibitem[K{\"u}nzel et~al.(2017)K{\"u}nzel, Sekhon, Bickel, and
  Yu]{kunzel2017meta}
S.~R. K{\"u}nzel, J.~S. Sekhon, P.~J. Bickel, and B.~Yu.
\newblock Meta-learners for estimating heterogeneous treatment effects using
  machine learning.
\newblock \emph{arXiv preprint arXiv:1706.03461}, 2017.

\bibitem[Lawlor et~al.(2004)Lawlor, Smith, and Ebrahim]{Lawlor04}
D.~A. Lawlor, G.~D. Smith, and S.~Ebrahim.
\newblock Commentary: The hormone replacement-coronary heart disease conundrum:
  is this the death of observational epidemiology?
\newblock \emph{International Journal of Epidemiology}, 2004.

\bibitem[Lee et~al.(2010)Lee, Lessler, and Stuart]{lee2010improving}
B.~K. Lee, J.~Lessler, and E.~A. Stuart.
\newblock Improving propensity score weighting using machine learning.
\newblock \emph{Statistics in medicine}, 29\penalty0 (3):\penalty0 337--346,
  2010.

\bibitem[Manski(2005)]{manski05}
C.~Manski.
\newblock \emph{Social Choice with Partial Knoweldge of Treatment Response}.
\newblock The Econometric Institute Lectures, 2005.

\bibitem[Manski(2003)]{manski2003partial}
C.~F. Manski.
\newblock \emph{Partial identification of probability distributions}.
\newblock Springer Science \& Business Media, 2003.

\bibitem[Masten and Poirier(2018)]{mp18}
M.~Masten and A.~Poirier.
\newblock Identification of treatment effects under conditional partial
  independence.
\newblock \emph{Econometrica}, 2018.

\bibitem[McCaffrey et~al.(2004)McCaffrey, Ridgeway, and
  Morral]{mccaffrey2004propensity}
D.~F. McCaffrey, G.~Ridgeway, and A.~R. Morral.
\newblock Propensity score estimation with boosted regression for evaluating
  causal effects in observational studies.
\newblock \emph{Psychological methods}, 9\penalty0 (4):\penalty0 403, 2004.

\bibitem[Nie and Wager(2017)]{nie2017learning}
X.~Nie and S.~Wager.
\newblock Learning objectives for treatment effect estimation.
\newblock \emph{arXiv preprint arXiv:1712.04912}, 2017.

\bibitem[Pagan and Ullah(1999)]{pagan1999nonparametric}
A.~Pagan and A.~Ullah.
\newblock \emph{Nonparametric econometrics}.
\newblock Cambridge university press, 1999.

\bibitem[Pedersen and Ottesen(2003)]{to-whi-03}
A.~T. Pedersen and B.~Ottesen.
\newblock Issues to debate on the women's health initiative (whi) study.
  epidemiology or randomized clinical trials-time out for hormone replacement
  therapy studies?
\newblock \emph{Human Reproduction}, 2003.

\bibitem[Prentice et~al.(2005)Prentice, Pettinger, and
  Anderson]{prentice-whi-05}
R.~L. Prentice, M.~Pettinger, and G.~L. Anderson.
\newblock Statistical issues arising in the women's health initiative.
\newblock \emph{Biometrics}, 2005.

\bibitem[Qian and Murphy(2011)]{qian2011performance}
M.~Qian and S.~A. Murphy.
\newblock Performance guarantees for individualized treatment rules.
\newblock \emph{Annals of statistics}, 39\penalty0 (2):\penalty0 1180, 2011.

\bibitem[Rosenbaum(2002)]{r02}
P.~Rosenbaum.
\newblock \emph{Observational Studies}.
\newblock Springer Series in Statistics, 2002.

\bibitem[Rossouw et~al.(2013)Rossouw, Manson, Kaunitz, and
  Anderson]{roussouw-whi-13}
J.~E. Rossouw, J.~E. Manson, A.~M. Kaunitz, and G.~L. Anderson.
\newblock Lessons learned from the women's health initiative trials of
  menopausal hormone therapy.
\newblock \emph{Obstetrics \& Gynecology.}, 2013.

\bibitem[Rubin(1980)]{rubin1980randomization}
D.~B. Rubin.
\newblock Comments on ``randomization analysis of experimental data: The fisher
  randomization test comment''.
\newblock \emph{Journal of the American Statistical Association}, 75\penalty0
  (371):\penalty0 591--593, 1980.

\bibitem[Shalit et~al.(2017)Shalit, Johansson, and
  Sontag]{shalit-johansson-sontag-17}
U.~Shalit, F.~Johansson, and D.~Sontag.
\newblock Estimating individual treatment effect: generalization bounds and
  algorithms.
\newblock \emph{Proceedings of the 34th International Conference on Machine
  Learning}, 2017.

\bibitem[Shapiro(2001)]{shapiro01}
A.~Shapiro.
\newblock \emph{Semi-Infinite Programming}, chapter On Duality Theory of Conic
  Linear Problems, pages 135--165.
\newblock 2001.

\bibitem[Swaminathan and Joachims(2015)]{sj15}
A.~Swaminathan and T.~Joachims.
\newblock Counterfactual risk minimization.
\newblock \emph{Journal of Machine Learning Research}, 2015.

\bibitem[Tan(2012)]{tan}
Z.~Tan.
\newblock A distributional approach for causal inference using propensity
  scores.
\newblock \emph{Journal of the American Statistical Associatioon}, 2012.

\bibitem[Wager and Athey(2017)]{wager2017estimation}
S.~Wager and S.~Athey.
\newblock Estimation and inference of heterogeneous treatment effects using
  random forests.
\newblock \emph{Journal of the American Statistical Association}, \penalty0
  (just-accepted), 2017.

\bibitem[Zhao et~al.(2017)Zhao, Small, and Bhattacharya]{zhaosmall17}
Q.~Zhao, D.~S. Small, and B.~B. Bhattacharya.
\newblock Sensitivity analysis for inverse probability weighting estimators via
  the percentile bootstrap.
\newblock \emph{ArXiv}, 2017.

\bibitem[Zhou et~al.(2017)Zhou, Mayer-Hamblett, Khan, and
  Kosorok]{zhou2017residual}
X.~Zhou, N.~Mayer-Hamblett, U.~Khan, and M.~R. Kosorok.
\newblock Residual weighted learning for estimating individualized treatment
  rules.
\newblock \emph{Journal of the American Statistical Association}, 112\penalty0
  (517):\penalty0 169--187, 2017.

\end{thebibliography}
\bibliographystyle{abbrvnat}

\newpage
\onecolumn
\begin{appendices}
\section{Population CATE sensitivity bounds}
\begin{lemma}\label{lemma: population-bounds}
	The sensitivity bounds for the conditional expected potential outcomes $\overline{\mu}_t(x)$ and $\underline{\mu}_t(x)$  defined in \eqref{eq: pop-outcome-upper}\eqref{eq: pop-outcome-lower} have the following equivalent characterization:
	\begin{align*}
	\overline{\mu}_t(x) &= 
	\sup\limits_{u \in \mathcal{U}^{nd} }
	\frac{\alpha_t(x)\int yf_t(y \mid x)dy + (\beta_t(x) - \alpha_t(x))\int u(y)yf_t(y \mid x)dy}{\alpha_t(x)\int f_t(y \mid x)dy+  (\beta_t(x) - \alpha_t(x))\int u(y)f_t(y \mid x)dy} \\
	\underline{\mu}_t(x) &= 
	\inf\limits_{u \in \mathcal{U}^{ni}}
	\frac{\alpha_t(x)\int yf_t(y \mid x)dy + (\beta_t(x) - \alpha_t(x))\int u(y)yf_t(y \mid x)dy}{\alpha_t(x)\int f_t(y \mid x)dy+  (\beta_t(x) - \alpha_t(x))\int u(y)f_t(y \mid x)dy}
	\end{align*}
	where 
	\begin{align*}
		\qquad \mathcal{U}^{nd} = \{u: \mathcal{Y} \to [0, 1] \mid u(y) \text{is nondecreasing}\}, \\ 
		\qquad \mathcal{U}^{ni} = \{u: \mathcal{Y} \to [0, 1] \mid u(y) \text{is nonincreasing}\},
	\end{align*}
	and 
%`
%`
%`
%`
%`
%`
	$\alpha_t(x)$ and $\beta(x)$ defined in \eqref{eq: uncertainty-set}.
\end{lemma}
\begin{proof}
	Recall that 
		\begin{align}
			\overline{\mu}_t(x) &= \sup_{w_t(y \mid x) \in [\alpha_t(x), \beta_t(x)]} \frac{\int yw_t(y \mid x)f_t(y \mid x)dy}{\int w_t(y \mid x)f_t(y \mid x)dy}, \\
			\underline{\mu}_t(x) &= \inf_{w_t(y \mid x) \in [\alpha_t(x), \beta_t(x)]} \frac{\int yw_t(y \mid x)f_t(y \mid x)dy}{\int w_t(y \mid x)f_t(y \mid x)dy}. 
		\end{align}
		By one-to-one change of variables $w_t(y \mid x) = \alpha_t(x) + u(y)(\beta_t(x) - \alpha_t(x))$ with $u: \mathcal{Y} \to [0, 1]$, 
		\begin{align}
		\overline{\mu}_t(x) &= 
		\sup\limits_{u: \mathcal{Y} \to [0, 1] }
		\frac{\alpha_t(x)\int yf_t(y, x)dy + (\beta_t(x) - \alpha_t(x))\int u(y)yf_t(y \mid x)dy}{\alpha_t(x)\int f_t(y, x)dy+  (\beta_t(x) - \alpha_t(x))\int u(y)f_t(y \mid x)dy} 
		\label{eq: pop-outcome-upper-3}  \\
		\underline{\mu}_t(x) &= 
		\inf\limits_{u: \mathcal{Y} \to [0, 1] }
		\frac{\alpha_t(x)\int yf_t(y, x)dy + (\beta_t(x) - \alpha_t(x))\int u(y)yf_t(y \mid x)dy}{\alpha_t(x)\int f_t(y, x)dy+  (\beta_t(x) - \alpha_t(x))\int u(y)f_t(y \mid x)dy} 
		\label{eq: pop-outcome-lower-3}
		\end{align}
	
	We next use duality to prove that the $u^*(y)$ that achieves the supremum in \eqref{eq: pop-outcome-upper-3} belongs to $\mathcal{U}^{nd}$. Similar result can be proved analogously for the infimum in  \eqref{eq: pop-outcome-lower-3}.  
	
	Denote that $a(x) = (\beta_t(x) - \alpha_t(x))$, $b(x) = (\beta_t(x) - \alpha_t(x))$, $c(x) = \alpha_t(x) \int y f_t(y \mid x)dy$, $d(x) = \alpha_t(x) \int f_t(y \mid x)dy$. Then the optimization problem in \eqref{eq: pop-outcome-upper-3} can be written as:
	\begin{align*}
	\max\limits_{\substack{ u: \mathcal{Y} \to [0, 1]}} \quad \frac{a(x)\ip{y}{{u}}_{f_t(y \mid x)} + c(x)}{b(x)\ip{1}{{u}}_{f_t(y \mid x)} + d(x)} \\
	\end{align*}
	where $\ip{\cdot}{\cdot}_{f_t(y \mid x)}$ is the inner product with respect to measure ${f_t(y \mid x)}$.
	
	By Charnes-Cooper transformation with $\tilde{u} = \frac{u}{b(x)\ip{1}{{u}}_{f_t(y \mid x)} + d(x)}$ and $\tilde{v}(x) = \frac{1}{b(x)\ip{1}{{u}}_{f_t(y \mid x)} + d(x)}$,  the optimization problem in \eqref{eq: pop-outcome-upper-3} is equivalent to the following linear program:
	\begin{align*}
	&\max\limits_{\substack{\tilde u: \mathcal{Y} \to [0, 1] \\ \tilde{v}(x)}} \quad a(x)\ip{y}{\tilde{u}}_{f_t(y \mid x)} + c(x)\tilde{v}(x) \\
	&\text{s.t.}  \quad \tilde{u}(y) \le \tilde{v}(x), -\tilde{u}(y) \le 0, \ \text{for} \ \forall y \in \mathcal{Y} \\ 
	& \quad b(x)\ip{1}{\tilde{u}}_{f_t(y \mid x)} + d(x)\tilde{v}(x) = 1, \tilde v(x) \ge 0
	\end{align*}
	
	Let the dual function $p(y)$ be associated with the primal constraint $\tilde{u}(y) \le \tilde{v}(x)$ ($u(y) \le 1$), and $q(y)$ be the dual function associated with $-\tilde{u}(y) \le 0$ ($u(y) \ge 0$), and $
	\lambda$ be the dual variable associated with the constraint $b(x)\ip{1}{\tilde{u}}_{f_t(y \mid x)} + d(x)\tilde{v} = 1$. The dual program is 
	\begin{align*}
	&\min\limits_{\substack{\lambda, p \succeq 0, q \succeq 0}} \quad \lambda  \\
	&\text{s.t.} \quad p - q + \lambda b(x)f_t(y\mid x) = a(x)yf_t(y\mid x) \\
	&\qquad -\ip{1}{p} + \lambda d(x) \ge c(x)
	\end{align*}
	
	By complementary slackness, at most one of $z_i$ or $\rho_i$ is nonzero. The first dual constraint implies that 
	\begin{align*}
	p &= (\beta_t(x) - \alpha_t(x))f_t(y \mid x)\max\{y - \lambda, 0\}, \\
	q &= (\beta_t(x) - \alpha_t(x))f_t(y \mid x)\max\{\lambda - y, 0\}.
	\end{align*}
	Moreover, the constraint that $-\ip{1}{p} + \lambda d(x) \ge c(x)$ should be tight at optimality. (otherwise there exists smaller yet feasible $\lambda$ that achives lower objective of the dual program.) This implies that 
	\[
		(\beta_t(x) - \alpha_t(x)) \int f_t(y \mid x)\max\{y - \lambda, 0\}dy = \alpha_t(x)\int(\lambda - y)f_t(y \mid x) dy
	\]
	This rules out the possibility that $\lambda > C_Y$ or $\lambda < -C_Y$ where $C_Y > 0$ such that $\abs{Y} \le C_Y$.  Thus $\exists y^H \in [-C_Y, C_y]$ such that when $y < y^H$, $q > 0$ so $u = 0$ and  when $y \ge y^H$, $p > 0$ so $u = 1$. Therefore, the optimal $u^*(y)$ that achieves the supremum in \eqref{eq: pop-outcome-upper-3} belongs to $\mathcal{U}^{nd}$. 
	
%`
%`
\end{proof}

\section{CATE sensitivity bounds estimators}
\begin{lemma}\label{lemma: sample-est-formulation}
	The kernel-regression based sensitivity bound estimators $\hat{\overline{\mu}}_t(x), \hat{\underline{\mu}}_t(x) $ given in \eqref{eq: est-upper-outcome}\eqref{eq: est-lower-outcome} have the following equivalent characterization: for $t \in \{0, 1\}$
	\begin{align*}
	\hat{\overline{\mu}}_t(x) &= \sup\limits_{u \in \mathcal{U}^{nd} }
	\frac{\sum_{i: T_i = t}^n \alpha (X_i)\mathbf{K}(\frac{X_i - x}{h})Y_i+ \sum_{i: T_i = t}^n (\beta(X_i) - \alpha (X_i))\mathbf{K}(\frac{X_i - x}{h})Y_iu(Y_i)}{\sum_{i: T_i = t}^n  \alpha (X_i)\mathbf{K}(\frac{X_i - x}{h})+ \sum_{i: T_i = t}^n  (\beta(X_i) - \alpha (X_i))\mathbf{K}(\frac{X_i - x}{h})u(Y_i)} \\
	\hat{\underline{\mu}}_t(x) &= \inf\limits_{u \in \mathcal{U}^{ni} }
	\frac{\sum_{i: T_i = t}^n \alpha (X_i)\mathbf{K}(\frac{X_i - x}{h})Y_i+ \sum_{i: T_i = t}^n (\beta(X_i) - \alpha (X_i))\mathbf{K}(\frac{X_i - x}{h})Y_iu(Y_i)}{\sum_{i: T_i = t}^n  \alpha (X_i)\mathbf{K}(\frac{X_i - x}{h})+ \sum_{i: T_i = t}^n  (\beta(X_i) - \alpha (X_i))\mathbf{K}(\frac{X_i - x}{h})u(Y_i)} \\
	\end{align*}
	where $\mathcal{U}^{nd} $ and $\mathcal{U}^{ni}$ are defined in Lemma \ref{lemma: population-bounds}.
	%`
	%`
	%`
	%`
	%`
	%`
\end{lemma}
\begin{proof}
	We prove the result for $\hat{\overline{\mu}}_t(x)$ and the result for $\hat{\underline{\mu}}_t(x) $ can be proved analogously. Given \eqref{eq: est-upper-outcome}, by one-to-one change of variable $W_i = \alpha(X_i) + (\beta(X_i) - \alpha(X_i))U_i$ where $U_i \in [0, 1]$, 
	\begin{align}
	\hat{\overline{\mu}}_t(x) &= \sup\limits_{U_i \in [0, 1]}
	\frac{\sum_{i: T_i = t}^n \alpha (X_i)\mathbf{K}(\frac{X_i - x}{h})Y_i+ \sum_{i: T_i = t}^n (\beta(X_i) - \alpha (X_i))\mathbf{K}(\frac{X_i - x}{h})Y_iU_i}{\sum_{i: T_i = t}^n  \alpha (X_i)\mathbf{K}(\frac{X_i - x}{h})+ \sum_{i: T_i = t}^n  (\beta(X_i) - \alpha (X_i))\mathbf{K}(\frac{X_i - x}{h})U_i}, \label{eq: est_outcome_upper3} \\
	\hat{\underline{\mu}}_t(x) &= \inf\limits_{U_i \in [0, 1]}
	\frac{\sum_{i: T_i = t}^n \alpha (X_i)\mathbf{K}(\frac{X_i - x}{h})Y_i+ \sum_{i: T_i = t}^n (\beta(X_i) - \alpha (X_i))\mathbf{K}(\frac{X_i - x}{h})Y_iU_i}{\sum_{i: T_i = t}^n  \alpha (X_i)\mathbf{K}(\frac{X_i - x}{h})+ \sum_{i: T_i = t}^n  (\beta(X_i) - \alpha (X_i))\mathbf{K}(\frac{X_i - x}{h})U_i}. \label{eq: est_outcome_lower3}
	\end{align}

	Now we use duality to prove that the optimal weights $U_i^*$ that attains the supremum in \eqref{eq: est_outcome_upper3} satisfies that $U_i^* = u(Y_i)$ for some function $u: \mathcal{Y} \to [0, 1]$ such that $u(y)$ is nondecreasing in $y$. The analogous result for \eqref{eq: est_outcome_lower3} can be proved similarly.
	
	Essentially, \eqref{eq: est_outcome_upper3} gives the following fractional linear program:
	\begin{align*} 
	\max\limits_{U} \frac{A^T U + C}{B^T U + D} \\
	\text{s.t. } 
	\begin{bmatrix}I_N \\-I_N\end{bmatrix}U 
	\leq \begin{bmatrix}1 \\ 0\end{bmatrix},
	\end{align*}
	where $U = [U_1, U_2, \dots, U_{n-1}, U_n]^\top$, $A = [a_1, a_2, \dots, a_n]^\top $ with $a_i = \mathbb{I}[T_i = t]  (\beta(X_i) - \alpha(X_i)) \mathbf{K}(\frac{X_i - x}{h})Y_i $, $B = [b_1, b_2, \dots, b_n]^\top$ with $b_i = \mathbb{I}[T_i = t]  (\beta(X_i) - \alpha(X_i)) \mathbf{K}(\frac{X_i - x}{h})$,   $C = \sum_{i: T_i = t}^n \alpha (X_i) \mathbf{K}(\frac{X_i - x}{h})Y_i$, and $D = \sum_{i: T_i = t}^n \alpha (X_i) \mathbf{K}(\frac{X_i - x}{h})$.
	
	By Charnes-Cooper transformation with $\tilde U = \frac{U}{B^\top U + D}$ and $\tilde V = \frac{1}{B^\top U + D}$, the linear-fractional program above is equivalent to the following linear program: 
	\begin{align*}
	\max\limits_{\tilde U, v} A^\top \tilde U + C \tilde{V}\\
	\text{s.t.}  \begin{bmatrix}
	I_n \\-I_n
	\end{bmatrix} \tilde U \leq \tilde V \begin{bmatrix}
	1 \\ 0 
	\end{bmatrix} \\
	B^\top \tilde U + \tilde V D = 1, \tilde V \geq 0 
	\end{align*}
	where the solution for $\tilde U, \tilde V$ yields a solution for the original program, $U = \frac{\tilde U}{\tilde V}$. 
	
	Let the dual variables $p_i \geq 0$ be associated with the primal constraints $\tilde U_i \le \tilde V$ 
	(corresponding to $U_i \leq 1$), $q_i \geq 0$ associated with $\tilde U_i \ge 0$ (corresponding to $U_i \geq 0$), and $\lambda$ associated with the constraint $B^\top \tilde U + D\tilde V  = 1$. Denote $P = [p_1, \dots, p_n]^\top$ and $Q = [q_1, \dots, q_n]^\top$.
	
	The dual problem is: 
	\begin{align*}
	&\qquad \qquad \min_{\lambda, z, \rho} \lambda \\
	&\text{s.t. } P - Q + \lambda B = A, p_i \geq 0, q_i \geq 0  \\
	& \qquad -1^TP + \lambda D \geq C
	\end{align*}
	
	%`
	
	By complementary slackness, at most one of $p_i$ or $q_i$ is nonzero. Rearranging the first set of equality constraints gives $p_i - q_i = \mathbb{I}(T_i = t)(\beta(X_i) - \alpha(X_i))\mathbf{K}(\frac{X_i - x}{h})(Y_i - \lambda)$, which implies that 
	\begin{align*}
	p_i &= \mathbb{I}[T_i = t] (\beta(X_i) - \alpha(X_i))\mathbf{K}(\frac{X_i - x}{h}) \max(Y_i - \lambda, 0)  \\
	q_i &= \mathbb{I}[T_i = t] (\beta(X_i) - \alpha(X_i))\mathbf{K}(\frac{X_i - x}{h}) \max(\lambda - Y_i, 0)
	\end{align*}
	Since the constraint $-1^TP + \lambda D \geq c$ is tight at optimality (otherwise there exists smaller yet feasible $\lambda$ that achives lower objective of the dual program), 
	\begin{align}
	\sum_{i=1}^n\mathbb{I}[T_i = t]\alpha(X_i)\mathbf{K}(\frac{X_i - x}{h})(\lambda - Y_i)  = \sum_{i = 1}^n \mathbb{I}[T_i = t] (\beta(X_i) - \alpha(X_i))\mathbf{K}(\frac{X_i - x}{h}) \max(Y_i - \lambda, 0) \label{eq: dual}
	\end{align}
	This rules out both $\lambda > \max_i{Y_i}$ and $\lambda < \min_i{Y_i}$, thus $Y_{(k)} < \lambda \leq Y_{(k+1)}$ for some $k$ where $Y_{(1)}, Y_{(2)}, \dots, Y_{(n)}$ are the order statistics of the sample outcomes . This means that $q_i > 0$  can happen only when $Y_i \le Y_{(k)}$, \ie, $U_i = 0$; and $p_i > 0$ can happen only when $i > k + 1$, \ie, $U_i = 1$. This proves there exist a nondecreasing function $u: \mathcal{Y} \to [0, 1]$ such that $U_i = u(Y_i)$ attains the upper bound in \eqref{eq: est_outcome_upper3}. 
\end{proof}

\begin{proof}[Proof for Proposition \ref{prop: sample-est-computation}]
	We prove the result for $\hat{\overline{\mu}}_t(x)$ and the result for $\hat{\underline{\mu}}_t(x)$ can be proved analogously. In the proof of Lemma \ref{lemma: sample-est-formulation}, \eqref{eq: dual} implies that $\exists k^H$ such that the optimal $Y_{k^H} <  \lambda^* \le Y_{k^H+1}$ and 
	\[
	\sum_{i \le k^H}\mathbb{I}[T_i = t]\alpha_t(X_i)\mathbf{K}(\frac{X_i - x}{h})(\lambda^* - Y_i) = \sum_{i \ge k^H+1}\mathbb{I}[T_i = t]\beta_t(X_i)\mathbf{K}(\frac{X_i - x}{h})(Y_i - \lambda^*).
	\]
	Thus 
	\begin{equation}\label{eq: kS^ctructure}
	\lambda^* = \frac{\sum_{i \le k^H}\mathbb{I}[T_i = t]\alpha_t(X_i)\mathbf{K}(\frac{X_i - x}{h})Y_i +  \sum_{i \ge k^H+1}\mathbb{I}[T_i = t]\beta_t(X_i)\mathbf{K}(\frac{X_i - x}{h})Y_i}{\sum_{i \le k^H}\mathbb{I}[T_i = t]\alpha_t(X_i)\mathbf{K}(\frac{X_i - x}{h}) +  \sum_{i \ge k^H+1}\mathbb{I}[T_i = t]\beta_t(X_i)\mathbf{K}(\frac{X_i - x}{h})}. 
	\end{equation}
	
	Now we prove that if $\lambda(k) \ge  \lambda(k+1)$, then $\lambda(k+1) \ge  \lambda(k+2)$, so $k^H = \inf\{k: \lambda(k) \ge \lambda(k+1)\}$. Note that $\lambda(k) \ge  \lambda(k+1)$ is equivalent to 
	\[
	\frac{\sum_{i \le k+1}\mathbb{I}[T_i = t]\alpha_t(X_i)\mathbf{K}(\frac{X_i - x}{h})Y_i + \sum_{i \ge k+2}\mathbb{I}[T_i = t]\beta_t(X_i)\mathbf{K}(\frac{X_i - x}{h})Y_i}{\sum_{i \le k+1}\mathbb{I}[T_i = t]\alpha_t(X_i)\mathbf{K}(\frac{X_i - x}{h}) + \sum_{i \ge k+2}\mathbb{I}[T_i = t]\beta_t(X_i)\mathbf{K}(\frac{X_i - x}{h})}  \le Y_{k+1}.
	\]
	Thus {\if1\forarxiv\footnotesize\fi
	\begin{align*}
	&\qquad\qquad\qquad \frac{\sum_{i \le k+2}\mathbb{I}[T_i = t]\alpha_t(X_i)\mathbf{K}(\frac{X_i - x}{h})Y_i + \sum_{i \ge k+3}\mathbb{I}[T_i = t]\beta_t(X_i)\mathbf{K}(\frac{X_i - x}{h})Y_i}{\sum_{i \le k+2}\mathbb{I}[T_i = t]\alpha_t(X_i)\mathbf{K}(\frac{X_i - x}{h}) + \sum_{i \ge k+3}\mathbb{I}[T_i = t]\beta_t(X_i)\mathbf{K}(\frac{X_i - x}{h})}   \\
	&\le \frac{\sum_{i \le k+1}\mathbb{I}[T_i = t]\alpha_t(X_i)\mathbf{K}(\frac{X_i - x}{h})Y_i + \sum_{i \ge k+2}\mathbb{I}[T_i = t]\beta_t(X_i)\mathbf{K}(\frac{X_i - x}{h})Y_i + (\beta_t(x) - \alpha_t(x))\mathbb{I}[T_i = t]\mathbf{K}(\frac{X_i - x}{h})Y_i}{\sum_{i \le k+1}\mathbb{I}[T_i = t]\alpha_t(X_i)\mathbf{K}(\frac{X_i - x}{h}) + \sum_{i \ge k+2}\mathbb{I}[T_i = t]\beta_t(X_i)\mathbf{K}(\frac{X_i - x}{h}) + (\beta_t(x) - \alpha_t(x))\mathbb{I}[T_i = t]\mathbf{K}(\frac{X_i - x}{h})}\\
	&\overset{(*)}{\le}  \frac{\bigg(\sum_{i \le k+1}\mathbb{I}[T_i = t]\alpha_t(X_i)\mathbf{K}(\frac{X_i - x}{h}) + \sum_{i \ge k+2}\mathbb{I}[T_i = t]\beta_t(X_i)\mathbf{K}(\frac{X_i - x}{h}) \bigg)Y_{k+1} + (\beta_t(x) - \alpha_t(x))\mathbb{I}[T_{k+2}= t]\mathbf{K}(\frac{X_{k+2}- x}{h})Y_{k+2}}{\sum_{i \le k+1}\mathbb{I}[T_i = t]\alpha_t(X_i)\mathbf{K}(\frac{X_i - x}{h}) + \sum_{i \ge k+2}\mathbb{I}[T_i = t]\beta_t(X_i)\mathbf{K}(\frac{X_i - x}{h}) + (\beta_t(x) - \alpha_t(x))\mathbb{I}[T_i = t]\mathbf{K}(\frac{X_i - x}{h})} \\
	&\le  \frac{\bigg(\sum_{i \le k+1}\mathbb{I}[T_i = t]\alpha_t(X_i)\mathbf{K}(\frac{X_i - x}{h}) + \sum_{i \ge k+2}\mathbb{I}[T_i = t]\beta_t(X_i)\mathbf{K}(\frac{X_i - x}{h}) \bigg)Y_{k+2} + (\beta_t(x) - \alpha_t(x))\mathbb{I}[T_{k+2}= t]\mathbf{K}(\frac{X_{k+2}- x}{h})Y_{k+2}}{\sum_{i \le k+1}\mathbb{I}[T_i = t]\alpha_t(X_i)\mathbf{K}(\frac{X_i - x}{h}) + \sum_{i \ge k+2}\mathbb{I}[T_i = t]\beta_t(X_i)\mathbf{K}(\frac{X_i - x}{h}) + (\beta_t(x) - \alpha_t(x))\mathbb{I}[T_i = t]\mathbf{K}(\frac{X_i - x}{h})} \\
	&= Y_{k+2},
	\end{align*}}
	where (*) holds due to \eqref{eq: kS^ctructure}. 
	
	This implies that $\lambda(k+1) \ge  \lambda(k+2)$. By strong duality, we know that $\hat{\overline{\mu}}_t(x) = \lambda^* = \overline\lambda(k^H; x)$ thus we prove the result for $\hat{\overline{\mu}}_t(x)$. We can analogously prove the result for $\hat{\underline{\mu}}_t(x)$.

\end{proof}

\begin{proof}[Proof for Theorem \ref{thm: consistent-cate}]
	Here we prove that $\hat{\overline{\mu}}_t(x) \to \overline{\mu}_t(x)$.  $\hat{\underline{\mu}}_t(x) \to \underline{\mu}_t(x)$ can be proved analogously. 

	Since $\int\mathcal K(u)du<\infty$,
	without loss of generality we assume $\int\mathcal K(u)=1$.
	
	Define the following quantities:{\if1\forarxiv\footnotesize\fi
	\begin{align*}
	\qquad \kappa_{\alpha}^y(t, x; n, h) &= \frac{1}{nh^d}\sum_{i: T_i = t}^n \alpha_t (X_i)\mathbf{K}(\frac{X_i - x}{h})Y_i,  \quad I_{\alpha}^y(t, x) =  \alpha_t(x)\int yf_t(y\mid x)dy, \\
	\qquad \kappa_{\beta- \alpha}^{u, y}(t, x; n, h) &=  \frac{1}{nh^d}\sum_{i: T_i = t}^n (\beta_t(X_i) - \alpha_t (X_i))\mathbf{K}(\frac{X_i - x}{h})Y_iu(Y_i),  \quad I_{\beta - \alpha}^{u, y}(t, x) = (\beta_t(x) - \alpha_t(x))\int u(y)yf_t(y \mid x)dy, \\
	\qquad \kappa_{\alpha}(t, x; n, h) &= \frac{1}{nh^d}\sum_{i: T_i = t}^n  \alpha_t (X_i)\mathbf{K}(\frac{X_i - x}{h}),  \quad I_{\alpha}(t, x) =  \alpha_t(x)\int f_t(y\mid x)dy, \\
	\qquad \kappa_{\beta- \alpha}^u(t, x; n, h) &=  \frac{1}{nh^d}\sum_{i: T_i = t}^n  (\beta_t(X_i) - \alpha_t (X_i))\mathbf{K}(\frac{X_i - x}{h})u(Y_i),  \quad I_{\beta - \alpha}^{u}(t, x) = (\beta_t(x) - \alpha_t(x))\int u(y)f_t(y \mid x)dy. 
	\end{align*}}
	Then 
	\begin{align*}
	\hat{\overline{\mu}}_t(x) &= \sup\limits_{u \in \mathcal{U}^{nd} }
	\frac{\kappa_{\alpha}^y(t, x; n, h)+ \kappa_{\beta - \alpha}^{u, y}(t, x; n, h)}{\kappa_{\alpha}(t, x; n, h)+ \kappa_{\beta - \alpha}^{u}(t, x; n, h)} \\
	\overline{\mu}_t(x) &= \sup\limits_{u \in \mathcal{U}^{nd} } \frac{I_{\alpha}^y(t, x) + I_{\beta-\alpha}^{u, y}(t, x)}{I_{\alpha}(t, x) + I_{\beta - \alpha}^{u}(t, x)}
	\end{align*}

	According to Lemma \ref{lemma: sup_inf}, {\if1\forarxiv\small\fi
	\begin{align}
	\abs{\hat{\overline{\mu}}_t(x) - \overline{\mu}_t(x)}  
	&\le \sup_{u \in \mathcal{U}^{nd}} 
	\bigg \vert 
	\frac{\kappa_{\alpha}^y(t, x; n, h)+ \kappa_{\beta - \alpha}^{u, y}(t, x; n, h)}{\kappa_{\alpha}(t, x; n, h)+ \kappa_{\beta - \alpha}^{u}(t, x; n, h)} - 	\frac{I_{\alpha}^y(t, x) + I_{\beta-\alpha}^{u, y}(t, x)}{I_{\alpha}(t, x) + I_{\beta - \alpha}^{u}(t, x)}
	\bigg \vert \nonumber
	\\
	&\le \sup_{u \in \mathcal{U}^{nd}} 
	\bigg\{ \abs{\kappa^y_{\alpha} + \kappa^{u, y}_{\beta-\alpha}}\frac{\abs{\kappa_{\alpha} + \kappa^{u}_{\beta - \alpha} - (I_{\alpha} + I^{u}_{\beta - \alpha})}}{\abs{\kappa_{\alpha} + \kappa^{u}_{\beta - \alpha}}\abs{I_{\alpha} + I^{u}_{\beta - \alpha}}} + \frac{1}{\abs{I_{\alpha} + I^{u}_{\beta - \alpha}}}\abs{\kappa^y_{\alpha} + \kappa^{u, y}_{\beta-\alpha}  - (I^y_{\alpha} + I^{u, y}_{\beta - \alpha})} \bigg\} \nonumber \\
	&\le \frac{(\Delta_1(t, x; n, h) + \abs{I^y_{\alpha} + I^{u, y}_{\beta-\alpha}})\Delta_2(t, x; n, h)}{\abs{I_{\alpha} + I^u_{\beta - \alpha}}(\abs{I_{\alpha} + I^u_{\beta - \alpha}} - \Delta_2(t, x; n, h))}+ \frac{\Delta_1(t,x; n, h)}{\abs{I_{\alpha} + I^u_{\beta - \alpha}}} \label{eq: error-decompose}
	\end{align}}
	where 
	\begin{align}
	\Delta_1(t, x; n, h) &= \sup_{u \in \mathcal{U}^{nd}}  \bigg\vert [\kappa^y_{\alpha}(t, x;n, h) + \kappa^{u, y}_{\beta - \alpha}(t, x; n, h)]   - [I_{\alpha}^y(t, x) + I_{\beta - \alpha}^{u, y}(t, x)]\bigg\vert ,  \\
	\Delta_2(t, x; n, h) &= \sup_{u \in \mathcal{U}^{nd}}  \bigg\vert [\kappa_{\alpha}(t, x; n, h) + \kappa^u_{\beta - \alpha}(t, x; n, h)]  - [I_{\alpha}(t, x) + I^u_{\beta - \alpha}(t, x)] \bigg\vert .
	\end{align}
	
	Therefore, we only need to prove that when $n \to \infty$, $h \to 0$, and $nh^{2d} \to \infty$, $\Delta_1(t, x; n, h) \overset{\text{p}}{\to} 0$ and $\Delta_2(t, x; n, h) \overset{\text{p}}{\to} 0$ for $t \in \{0, 1\}$ and $x \in \mathcal{X}$. We prove $\Delta_1(t, x; n, h) \overset{\text{p}}{\to} 0$ in this proof. $\Delta_2(t, x; n, h) \overset{\text{p}}{\to} 0$ can be proved analogously. 
	
	Note that 
	\begin{align*}
	\Delta_1(t, x; n, h) &\le \bigg\vert \kappa^y_{\alpha}(t, x;n, h) - I_{\alpha}^y(t, x) \bigg \vert + \sup_{u \in \mathcal{U}^{nd}}  \bigg\vert \kappa^{u, y}_{\beta - \alpha}(t, x; n, h) -  I^{u, y}_{\beta - \alpha}(t, x)  \bigg\vert.
	\end{align*}
	
	\textbf{Step 1}: prove that $\sup_{u \in \mathcal{U}^{nd}}  \bigg\vert \kappa^{u, y}_{\beta - \alpha}(t, x; n, h) -  I^{u, y}_{\beta - \alpha}(t, x)  \bigg\vert \to 0$.
	
	Obviously 
	\begin{align*}
	&\sup_{u \in \mathcal{U}^{nd}}  \bigg\vert \kappa^{u, y}_{\beta - \alpha}(t, x; n, h) -  I^{u, y}_{\beta - \alpha}(t, x)  \bigg\vert \\
	&\le \sup_{u \in \mathcal{U}^{nd}}  \bigg\vert \kappa^{u, y}_{\beta - \alpha}(t, x; n, h) - \expect \kappa^{u, y}_{\beta - \alpha}(t, x; n, h) \bigg\vert +  \sup_{u \in \mathcal{U}^{nd}}  \bigg\vert \expect \kappa^{u, y}_{\beta - \alpha}(t, x; n, h) -  I^{u, y}_{\beta - \alpha}(t, x)  \bigg\vert \\
	&:= \Lambda_1 + \Lambda_2 
	\end{align*}
	
	\textbf{Step 1.1}: prove $\Lambda_1 \overset{\text{p}}{\to} 0$.
	
	By assumption,
	there exists $\delta>0$ with $e_t(x,y)\in[\delta,1-\delta]$. Hence,
	$\alpha_t(x) \le C_{\delta, \Gamma}(\alpha) = \frac{1}{\Gamma}(\frac{1}{\delta} - 1) + 1$ and $\beta_t(x) - \alpha_t(x) \le C_{\delta, \Gamma}(\beta - \alpha) = (\Gamma - \frac{1}{\Gamma})(\frac{1}{\delta} - 1)$. Under the assumptions that $\abs{{K}(x)} \le C_K$ and $\abs{Y} \le C_Y$, there exists a constant $c >0$ such that for any two different observations $(X_i, T_i, Y_i)$ and $(X'_i, T'_i, Y'_i)$, 
	\begin{align*}
	\bigg \vert &\frac{1}{nh^d} (\beta_t(X_i) - \alpha_t(X_i))\mathbf{K}(\frac{X_i - x}{h})\mathbb{I}(T_i = t)u(Y_i)Y_i  \\
	&- \frac{1}{nh^d} (\beta_t(X'_i) - \alpha_t(X'_i))\mathbf{K}(\frac{X'_i - x}{h})\mathbb{I}(T'_i = t)u(Y'_i)Y'_i \bigg \vert \\
	& \le \frac{cC_K^dC_YC_{\delta, \Gamma}(\beta - \alpha)}{nh^d}.
	\end{align*}
	Then Lemma \ref{lemma: sup_diff} and Mcdiarmid inequality implies that with high probability at least $1 - \exp(-\frac{2nh^{2d}\epsilon^2}{c^2C_Y^2C_K^{2d}C^2_{\delta, \Gamma}(\beta - \alpha)})$, 
	\[
	\Lambda_1 \le \expect \Lambda_1 + \epsilon.
	\]
	Moreover, we can bound  $\expect \Lambda_1$ by Rademacher complexity: for i.i.d Rademacher random variables $\sigma_1, \dots, \sigma_n$, 
	\begin{align}
	\expect \Lambda_1 \le 2\expect \sup_{u \in \mathcal{U}^{nd}} \bigg\vert\frac{1}{nh^d}\sum_{i = 1}^n \sigma_i  (\beta_t(X_i) - \alpha_t(X_i)) \mathbf{K}(\frac{X_i - x}{h})\mathbb{I}(T_i = t)u(Y_i)Y_i \bigg\vert. \label{eq: rademacher}
	\end{align}
	
	Furthermore, we can bound the Rademacher complexity given the monotonicity structure of $\mathcal{U}^{nd}$. Suppose we reorder the data so that $Y_1 \le Y_2 \le \dots \le Y_n$. Denote the whole sample by $\mathcal{S} = \{(X_i, T_i, Y_i)\}_{i = 1}^n$ and $\kappa_i = (\beta_t(X_i) - \alpha_t(X_i))\mathbf{K}(\frac{X_i - x}{h})\mathbb{I}(T_i = t)u(Y_i)Y_i $. Since \eqref{eq: rademacher} is a linear programming problem, we only need to consider the vertex solutions, \ie, $u \in \{0, 1\}$.  Therefore, 
	\begin{align}
	\expect \Lambda_1 \le 2\expect \sup_{u \in \mathcal{U}^{nd}, u \in \{0, 1\}} \bigg\vert\frac{1}{nh^d}\sum_{i = 1}^n \sigma_i  (\beta_t(X_i) - \alpha_t(X_i)) \mathbf{K}(\frac{X_i - x}{h})\mathbb{I}(T_i = t)u(Y_i)Y_i \bigg\vert. 
	\end{align}
	
	Conditionally on $\mathcal{S}$, $(u(Y_1), \dots, u(Y_n))$ thus can only have $n+1$ possible values: 
	\begin{align*}
	(0, 0, \dots, 0, 0), (0, 0, \dots, 0, 1), \dots, \\
	(0, 1, \dots, 1, 1), (1, 1, \dots, 1, 1).
	\end{align*}
	This means that conditionally on $\mathcal{S}$, $(\kappa_1, \kappa_2, \dots, \kappa_N)$ can have at most $N+1$ possible values. Plus, $\abs{\kappa_i} \le C_K^dC_YC_{\delta, \Gamma}(\beta - \alpha)$. So by Massart's finite class lemma, 
	\begin{align*}
	&\expect \sup_{u \in \mathcal{U}^{nd}} \bigg\vert\frac{1}{nh^d}\sum_{i = 1}^N \sigma_i (\beta_t(X_i) - \alpha_t(X_i))\mathbf{K}(\frac{X_i - x}{h})\mathbb{I}(T_i = t)u(Y_i)Y_i\bigg\vert \\
	& \le \sqrt{\frac{2 C_Y^2C_K^{2d}C^2_{\delta, \Gamma}(\beta - \alpha)\log(n+1)}{nh^{2d}}}. 
	\end{align*}
	Therefore, with high probability at least $1 - \exp(-\frac{2nh^{2d}\epsilon^2}{c^2C_Y^2C_K^{2d}C^2_{\delta, \Gamma}(\beta - \alpha)})$,
	\[
	\Lambda_1 \le 2\sqrt{\frac{2 C_Y^2C_K^{2d}C^2_{\delta, \Gamma}(\beta - \alpha)\log(n+1)}{nh^{2d}}} + \epsilon,
	\]
	which means that $\Lambda_1 \overset{\text{p}}{\to} 0$ when $nh^{2d} \to \infty$. 
	
	\textbf{Step 1.2}: prove $\Lambda_2 \overset{\text{p}}{\to} 0$. 
	\begin{align*}
	& \expect \frac{1}{nh^d}\sum_{i = 1}^n (\beta(X_i) - \alpha(X_i))\mathbf{K}(\frac{X_i - x}{h})\mathbb{I}(T_i = t)u(Y_i)Y_i \\ 
	&= \frac{1}{h^d}\expect[(\beta_t(X_i) - \alpha_t(X_i))\mathbf{K}(\frac{X_i - x}{h})\mathbb{I}(T_i = t)u(Y_i)Y_i] \\
	&= \frac{1}{h^d}\int u(y)y \bigg[\int (\beta_t(z') - \alpha_t(z'))\mathbf{K}(\frac{z' - x}{h})f_t(y \mid x) dz' \bigg] dy \\
	&\overset{(a)}{=} \int u(y)y \bigg[\int (\beta_t - \alpha_t)(x + zh)\mathbf{K}(z)f_t(y \mid x+zh) dz \bigg] dy
	\end{align*}
	where in (a) we use change-of-variable $z = \frac{z' - x}{h}$. 
	
	Since $\beta _t(x)$, $\alpha_t(x)$, and $f_t(y \mid x)$ are twice continuously differentiable with respect to $x$ at any $x \in \mathcal{X}$ and $y \in \mathcal{Y}$.   Apply Taylor expansion to $(\beta_t - \alpha_t)(x + zh)$ and $f_t(y \mid x+zh)$ around $x$:
	\begin{align*}
	(\beta_t - \alpha_t)(x + zh) &= (\beta_t - \alpha_t)(x) + hz^\top\frac{d}{dx}(\beta_t - \alpha_t)(x) + \frac{1}{2}h^2z^\top\frac{d^2}{dx^2}(\beta_t - \alpha_t)(x)z + o(h^2) \\
	f_t(y \mid x+zh)  &= f_t(y \mid x)  + hz^\top\frac{\partial }{\partial x}f_t(y \mid x)  + \frac{1}{2}h^2z^\top\frac{\partial^2 }{\partial x^2}f_t(y \mid x)z + o(h^2) \\  
	\end{align*}
	Then 
	\begin{align*}
	&\expect \frac{1}{nh^d}\sum_{i = 1}^n (\beta_t(X_i) - \alpha_t(X_i))\mathbf{K}(\frac{X_i - x}{h})\mathbb{I}(T_i = t)u(Y_i)Y_i \\
	&= \int (\beta_t-\alpha_t)(x)u(y)yf_t(y \mid x)dy  + \frac{h^2}{2}\big(\int \mathbf{K}(z)z^2dz \big)\int u(y)y\bigg(f_t(y \mid x)\frac{d^2}{dx^2}(\beta_t - \alpha_t)(x) \\
	&\qquad\qquad\qquad\qquad\qquad\qquad + (\beta_t - \alpha_t)(x)\frac{\partial^2 }{\partial x^2}f_t(y \mid x) + 2\frac{d}{dx}(\beta_t(x) - \alpha_t(x))\frac{\partial }{\partial x}f_t(y \mid x) \bigg)dy  + o(h^2) \\
	&=  \frac{h^2}{2}\big(\int \mathbf{K}(z)z^2dz\big)\int u(y)y\bigg(\frac{\partial^2 }{\partial x^2}\big((\beta_t-\alpha_t)(x)f_t(y \mid x)\big)\bigg)dy + I^u_{\beta - \alpha}(t, x) + o(h^2)
	\end{align*}
	
	Since  the first order and second order derivatives of $\beta _t(x)$, $\alpha_t(x)$, and $f_t(y \mid x)$ with respect to $x$ are bounded, obviously, 
	\[
	\abs{\int u(y)y\bigg(\frac{\partial^2 }{\partial x^2}\big((\beta_t-\alpha_t)(x)f_t(y \mid x)\big)\bigg)dy} < \infty.
	\]
	
	Thus as $h \to 0$, 
	\begin{align*}
	\Lambda_2 &= \sup_{u \in \mathcal{U}^{nd}}\bigg \vert \frac{h^2}{2}\big(\int \mathbf{K}(z)z^2dz\big)\int u(y)y\bigg(\frac{\partial^2 }{\partial x^2}\big((\beta_t-\alpha_t)(x)f_t(y \mid x)\big)\bigg)dy + o(h^2)\bigg \vert \\
	&\to 0.
	\end{align*}
	
	\textbf{Step 2}: prove that $\bigg\vert \kappa^y_{\alpha}(t, x;n, h) - I_{\alpha}^y(t, x) \bigg \vert  \overset{\text{p}}{\to} 0$.
	Obviously 
	\begin{align*}
	&\bigg\vert \kappa^{y}_{\alpha}(t, x; n, h) -  I^y_{\alpha}(t, x)  \bigg\vert \\
	&\le \bigg\vert \kappa^{y}_{\alpha}(t, x; n, h) - \expect \kappa^{y}_{\alpha}(t, x; n, h) \bigg\vert +    \bigg\vert \expect \kappa^{y}_{\alpha}(t, x; n, h) -  I^y_{\alpha}(t, x)  \bigg\vert \\
	&:= \Lambda_3 + \Lambda_4 
	\end{align*}
	\textbf{Step 2.1}: prove $\Lambda_3 \overset{\text{p}}{\to} 0$.
	By Mcdiarmid inequality, with high probability at least $1 - 2\exp(-\frac{2nh^{2d}\epsilon^2}{c^2C_Y^2C_K^2C^2_{\delta, \Gamma}(\alpha)})$,
	\[
	\Lambda_3 \le  \epsilon.
	\]
	Thus $\Lambda_3 \overset{\text{p}}{\to} 0$ when $nh^{2d} \to \infty$. 
	
	\textbf{Step 2.2}: prove $\Lambda_4 \overset{\text{p}}{\to} 0$.
	Similarly to Step 1.2, we can prove that 
	\begin{align*}
	\Lambda_4 &=\bigg \vert \frac{h^2}{2}\big(\int \mathbf{K}(z)z^2dz\big)\int u(y)y\frac{\partial^2 }{\partial x^2}\big(\alpha_t(x)f_t(y \mid x)\big)dy + o(h^2)\bigg \vert \\
	&\to 0.
	\end{align*}
	
	So far, we have proved that $\Delta_1(t, x; n, h) \overset{\text{p}}{\to} 0$. Analogously we can prove that $\Delta_2(t, x; n, h) \overset{\text{p}}{\to} 0$. Thus when $n \to \infty$, $h \to 0$, and $nh^{2d} \to \infty$, 
	\[
	\abs{\hat{\overline{\mu}}_t(x) - \overline{\mu}_t(x)} \overset{\text{p}}{\to} 0. 
	\] 
	Analogously, 
	\[
	\abs{\hat{\underline{\mu}}_t(x) - \underline{\mu}_t(x)} \overset{\text{p}}{\to} 0. 
	\] 
	Therefore, 
	\[
	\hat{\overline{\tau}}(x) \overset{\text{p}}{\to} {\overline{\tau}}(x),    \quad \hat{\underline{\tau}}(x) \overset{\text{p}}{\to} {\underline{\tau}}(x). 
	\]
\end{proof}

\begin{lemma} \label{lemma: sup_inf}
	For functions $J: S \to \mathbb{R}$ and $\tilde{J}: S \to \mathbb{R}$ where $S$ is some subset in Euclidean space, 
	\begin{align*}
	\big \vert \sup_{x \in S} J(x) - \sup_{x \in S}\tilde{J}(x)\big \vert &\le \sup_{x \in S} |J(x) - \tilde{J}(x)| \\
	\big \vert \inf_{x \in S} J(x) - \inf_{x \in S}\tilde{J}(x)\big \vert &\le \sup_{x \in S} |J(x) - \tilde{J}(x)|
	\end{align*}
	
\end{lemma}
\begin{proof}
	Obviously,
	\begin{align*}
	\sup_{x \in S} J(x) &\le \sup_{x \in S} \tilde{J}(x) + \sup_{x \in S} \{J(x) -  \tilde{J}(x)\} \\
	\inf_{x \in S} -\tilde{J}(x) &\ge \inf_{x \in S} -J(x) +  \inf_{x \in S} \{\tilde{J}(x)-J(x)\}.
	\end{align*}
	This implies that 
	\begin{align*}
	- \sup_{x \in S} |J(x)-\tilde{J}(x)| &\le \inf_{x \in S} \{\tilde{J}(x)-J(x)\}\\
	& \le \sup_{x \in S} J(x) - \sup_{x \in S} \tilde{J}(x) \\
	& \le \sup_{x \in S} \{J(x) -  \tilde{J}(x)\} \le \sup_{x \in S} |J(x) -  \tilde{J}(x)|, 
	\end{align*}
	\ie, $|\sup_{x \in S} J(x) - \sup_{x \in S} \tilde{J}(x)| \le \sup_{x \in S} |J(x) -  \tilde{J}(x)|$.
	
	On the other hand, 
	\begin{align*}
	\sup_{x \in S} -\tilde{J}(x) &\le \sup_{x \in S} -J(x) + \sup_{x \in S} \{J(x) - \tilde{J}(x)\} \\
	\inf_{x \in S} J(x) &\ge \inf_{x \in S} \tilde{J}(x) +  \inf_{x \in S} \{J(x)-\tilde{J}(x)\}
	\end{align*}
	which implies that 
	\begin{align*}
	- \sup_{x \in S} |J(x)-\tilde{J}(x)| &\le -\sup_{x \in S} \{\tilde{J}(x)-J(x)\} \\
	&\le \inf_{x \in S} J(x) - \inf_{x \in S} \tilde{J}(x) \\ 
	&\le \sup_{x \in S} \{J(x) -  \tilde{J}(x)\} \le  \sup_{x \in S} |J(x) -  \tilde{J}(x)|
	\end{align*}
	Namely $|\inf_{x \in S} J(x) - \inf_{x \in S} \tilde{J}(x)| \le \sup_{x \in S} |J(x) -  \tilde{J}(x)|$.
\end{proof}

\begin{lemma} \label{lemma: sup_diff}
	For functions $J: S \to \mathbb{R}$ and $\tilde{J}: S \to \mathbb{R}$ where $S$ is some subset in Euclidean space, 
	\[
	\big \vert sup_{x \in S} |J(x)| - sup_{x \in S}|\tilde{J}(x)|\big \vert \le \sup_{x \in S} |J(x) - \tilde{J}(x)|
	\]
	\begin{proof}
		On the one hand, 
		\[
		\sup_{x \in S} |J(x)| = \sup_{x \in S} |J(x) - \tilde{J}(x) + \tilde{J}(x)| \le \sup_{x \in S} |J(x) - \tilde{J}(x)| + \sup_{x \in S}|\tilde{J}(x)|, 
		\]
		which implies that 
		\[
		\sup_{x \in S} |J(x)| - \sup_{x \in S}|\tilde{J}(x)| \le \sup_{x \in S} |J(x) - \tilde{J}(x)|.
		\]
		On the other hand, 
		\[
		\inf_{x \in S} -|\tilde{J}(x)| = \inf_{x \in S} -|\tilde{J}(x) - J(x) + J(x)| \ge \inf_{x \in S} (-|\tilde{J}(x) - J(x)| - |J(x)|) \ge  \inf_{x \in S} (-|\tilde{J}(x) - J(x)|) + \inf_{x \in S} (- |J(x)|),
		\]
		which implies that 
		\[
		\sup_{x \in S} |J(x)| - \sup_{x \in S}|\tilde{J}(x)| \ge -\sup_{x \in S} |J(x) - \tilde{J}(x)|.
		\]
		Therefore, the conclusion follows. 
	\end{proof}
\end{lemma}

\section{Policy Learning}
\begin{proof}[Proof for Proposition \ref{prop: population-opt-policy}]
	The optimal policy ${\pi}^*(\cdot; \Gamma)$ solves the following optimization problem:
	
	\begin{equation}\label{eq: minimax}
	\inf\limits_{\pi: \mathcal{X} \to [0, 1]} 
	\sup_{\tau(x)\in\mathcal T(x;\Gamma)~\forall x\in\mathcal X}\expect[(\pi(X) - \pi_0(X))\tau(X)].
	\end{equation}
	
	Since both $\pi$ and $\tau$ are bounded, according to Von Neumann theorem, the optimaization problem \eqref{eq: minimax} is equivalent to 
	\begin{equation}\label{eq: maximin}
	\sup\limits_{\tau \in {\mathcal{T}}}\inf\limits_{\pi: \mathcal{X} \to [0, 1]} \expect[(\pi(X) - \pi_0(X))\tau(X)],
	\end{equation}
	which means that there exist optimal $\tau^*\in\mathcal T$ and $\pi^*$
	such that:
	(a) $\tau^*$ is pessimal for $\pi^*$ in that 
	$\expect[({\pi^*(X)-\pi_0(X)}){\tau^*(X)}] \geq \expect[({\pi^*(X)-\pi_0(X)}){\tau(X)}]$ for $\forall\tau\in\mathcal T$
	and (b) $\pi^*$ is optimal for $\tau^*$ in that
	$\expect[({\pi^*(X)-\pi_0(X)}){\tau^*(X)}]  \leq \expect[({\pi(X)-\pi_0(X)}){\tau^*(X)}],\ \forall\pi: \mathcal{X} \to [0,1]$.
	Obviously (b) implies that $\pi^*=\indic{\tau^*(x)<0} + \pi_0(x)\indic{\tau^*(x) = 0}$ can make an optimal policy. Plugging $\pi^*=\indic{\tau^*(x)<0} + \pi_0(x)\indic{\tau^*(x) = 0}$ into \eqref{eq: maximin} gives 
	\begin{equation}\label{eq: opt-tau}
	\tau^* = \argmax_{\tau \in \mathcal T}\expect \min((1 - \pi_0(X))\tau(X), -\pi_0(X)\tau(X)).
	\end{equation}
	
	Actually $\tau^*$ has closed form solution:
	\begin{outline}
		\1 When $\overline{\tau}(x) \le 0$,  obviously $\tau(x) \le 0$ so $\min((1 - \pi_0(x))\tau(x), -\pi_0(x)\tau(x)) = (1 - \pi_0(x))\tau(x)$. 
		\2  $\tau^*(x) = \overline{\tau}(x)$ if $\pi_0(x) < 1$;
		\2 $\tau^*(x)$ can be anything between $\underline{\tau}(x)$ and $\overline{\tau}(x)$ if $\pi_0(x) = 1$.
		\1 When $\underline{\tau}(x) \ge 0$, obviously $\tau(x) \ge 0$ so $\min((1 - \pi_0(x))\tau(x), -\pi_0(x)\tau(x)) = -\pi_0(x)\tau(x)$.
		\2  $\tau^*(x) = \underline{\tau}(x)$ if $\pi_0(x) > 0$;
		\2 $\tau^*(x)$ can be anything between $\underline{\tau}(x)$ and $\overline{\tau}(x)$ if $\pi_0(x) = 0$.
		\1 When $\underline{\tau}(x) < 0 < \overline{\tau}(x)$, 
		\2 If $0 < \pi_0(x) < 1$, when choosing $\tau^*(x) \ge 0$, $\min((1 - \pi_0(x))\tau^*(x), -\pi_0(x)\tau^*(x)) = -\pi_0(x)\tau^*(x)) \le 0$, so $\tau^*(x)$ must be $0$; similarly, when choosing $\tau^*(x) \le 0$, $\tau^*(x)$ must be $0$. This means that $\tau^*(x) = 0$. 
		\2 When $\pi_0(x) = 0$, $\tau^*(x)$ can be anything between $0$ and $\overline\tau(x)$.
		\2 When $\pi_0(x) = 1$, $\tau^*(x)$ can be anything between $\underline\tau(x)$ and $0$.
	\end{outline}
	
	In summary, the following $\tau^*$ always solves the optimization problem in \eqref{eq: opt-tau}: 
	\begin{equation*}
	\tau^*(x) = \overline\tau(x)\ind(\overline\tau(x) \le 0) +  \underline\tau(x)\ind(\underline\tau(x) \ge 0).
	\end{equation*}
	
	Therefore, the following policy is a minimax-optimal policy that optimizes \eqref{eq: maximin}: 
	\begin{equation*}
	\pi^*(x)=\indic{\tau^*(x) < 0} + \pi_0(x)\indic{\tau^*(x) = 0},
	\end{equation*}
	with 
	\begin{equation*}
	\tau^*(x) = \overline\tau(x)\ind(\overline\tau(x) \le 0) +  \underline\tau(x)\ind(\underline\tau(x) \ge 0).
	\end{equation*}
	
	Namely, 
	\begin{equation*}
	\pi^*(x) = \ind(\overline{\tau}(x) \le 0) +  \pi_0(x)\ind(\underline{\tau}(x) < 0 \le \overline{\tau}(x)).
	\end{equation*}
	
\end{proof}

%`
%`
%`

\begin{proof}[Proof for Theorem \ref{thm: policy}]
	According to the proof for Proposition \ref{prop: population-opt-policy},

	\begin{align*}
	\sup\limits_{\tau \in \mathcal{T}}\overline R_{\pi_0}(\pi^*(\cdot;\Gamma);\Gamma) &= \expect \min((1 - \pi_0(X))\tau^*(X), -\pi_0(X)\tau^*(X)) \\
	&= \expect (1 - \pi_0(X))\overline{\tau}(X)\ind(\overline{\tau}(X) < 0)  + \expect ( - \pi_0(X))\underline{\tau}(X)\ind(\underline{\tau}(X) > 0)
	\end{align*}
	In contrast, 
	\begin{align*}
	\sup\limits_{\tau \in \mathcal{T}}\overline R_{\pi_0}(\hat{\pi}^*(\cdot;\Gamma);\Gamma)&= \max_{\tau \in \mathcal T} \expect [(1 - \pi_0(X))\tau(X) \ind(\hat{\overline\tau}(X) < 0) + ( - \pi_0(X))\tau(X) \ind(\hat{\underline\tau}(X) > 0)] \\
	&=  \expect [(1 - \pi_0(X))\overline\tau(X) \ind(\hat{\overline\tau}(X) < 0)  + ( - \pi_0(X))\underline\tau(X) \ind(\hat{\underline\tau}(X) > 0)]
	\end{align*}
	Thus {\if1\forarxiv\small\fi
	\begin{align*}
	\sup\limits_{\tau \in \mathcal{T}}\overline R_{\pi_0}(\hat{\pi}^*(\cdot;\Gamma);\Gamma) - \sup\limits_{\tau \in \mathcal{T}}\overline R_{\pi_0}(\pi^*(\cdot;\Gamma);\Gamma)	&=\expect\bigg[\bigg(  (1 - \pi_0(X))\overline{\tau}(X)\ind(\overline{\tau}(X) < 0)  +  ( - \pi_0(X))\underline{\tau}(X)\ind(\underline{\tau}(X) > 0)\bigg) \\
	&\qquad -  \bigg( (1 - \pi_0(X))\overline\tau(X) \ind(\hat{\overline\tau}(X) < 0) + ( - \pi_0(X))\underline\tau(X) \ind(\hat{\underline\tau}(X) > 0) \bigg) \bigg] \\
	&=\expect\bigg[  (1 - \pi_0(X))\overline{\tau}(X)\bigg(\ind(\overline{\tau}(X) < 0) -  \ind(\hat{\overline\tau}(X) < 0)\bigg)\\
	&\qquad\qquad +  ( - \pi_0(X))\underline{\tau}(X)\bigg(\ind(\underline{\tau}(X) > 0) -  \ind(\hat{\underline\tau}(X) > 0) \bigg)  \bigg] \\
	&= -\expect\bigg[\bigg( (1 - \pi_0(X))\abs{\overline{\tau}(X)} \ind(\op{sign}(\overline{\tau}(X)) \neq \op{sign}(\hat{\overline\tau}(X)))\bigg)  \\
	&  \qquad\qquad\qquad +\bigg(( - \pi_0(X))\abs{\underline{\tau}(X)}\ind( \op{sign}(\underline{\tau}(X)) \neq \op{sign}(\hat{\underline\tau}(X)))\bigg)   \bigg] \\
	\end{align*}}
	
	Next, we prove that under the assumptions in Theorem \ref{thm: consistent-cate}, when $n \to \infty$, $h \to 0$, and $nh^2 \to \infty$,
	\[
	\expect\bigg( (1 - \pi_0(X))\abs{\overline{\tau}(X)} \ind(\op{sign}(\overline{\tau}(X)) \neq \op{sign}(\hat{\overline\tau}(X)))\bigg) \to 0.
	\]
	Given that $|Y| \le C_Y$, $\abs{\overline{\tau}(x)} \le 2C_Y$ and $\abs{\underline{\tau}(x)} \le 2C_Y$. For any $\eta > 0$, 
	\begin{align*}
	&\qquad\qquad\qquad\qquad  \expect\left[  (1 - \pi_0(X))\abs{\overline{\tau}(X)} \ind(\op{sign}(\overline{\tau}(X)) \neq \op{sign}(\hat{\overline\tau}(X)))\right] \\
	&\le 2C_Y \pr \bigg( \ind(\op{sign}(\overline{\tau}(X)) \neq \op{sign}(\hat{\overline\tau}(X)))\ind (|\overline{\tau}(X)| > \eta)\bigg)  + \eta \pr \bigg( \ind(\op{sign}(\overline{\tau}(X)) \neq \op{sign}(\hat{\overline\tau}(X)))\ind (|\overline{\tau}(X)| \le \eta)\bigg) \\
	&\overset{(b)}{\le} 2C_Y \pr \bigg( \ind(\op{sign}(\overline{\tau}(X)) \neq \op{sign}(\hat{\overline\tau}(X)))\ind (|\overline{\tau}(X) - \hat{\overline{\tau}}(X)| > \eta)\bigg) + \eta\\
	&\le 2C_Y \pr(|\overline{\tau}(X) - \hat{\overline{\tau}}(X)| > \eta) + \eta \\
	&= 2C_Y \expect\bigg[\pr(|\overline{\tau}(X) - \hat{\overline{\tau}}(X)| > \eta \mid X)\bigg] + \eta \overset{(c)}{\to} \eta
	\end{align*}
	Here (b) holds because when $\op{sign}(\overline{\tau}(X)) \neq \op{sign}(\hat{\overline\tau}(X))$, $|\overline{\tau}(X) - \hat{\overline{\tau}}(X)| > |\overline{\tau}(X)|$; (c) holds because Theorem \ref{thm: consistent-cate} proves that  $\pr(|\overline{\tau}(X) - \hat{\overline{\tau}}(X)| > \eta \mid X)  \to 0$ , which implies $\expect\bigg[\pr(|\overline{\tau}(X) - \hat{\overline{\tau}}(X)| > \eta \mid X)\bigg] \to 0$ according to bounded convergence theorem considering that $\pr(|\overline{\tau}(X) - \hat{\overline{\tau}}(X)| > \eta \mid X) \le 1$.
	
	Therefore, when $n\to \infty$, $h \to 0$ and $nh^{2d} \to \infty$,
	\[\expect\left[  (1 - \pi_0(X))\abs{\overline{\tau}(X)} \ind(\op{sign}(\overline{\tau}(X)) \neq \op{sign}(\hat{\overline\tau}(X)))\right]  \to 0.
	\]
	
	Analogously, we can prove that, when $n\to \infty$, $h \to 0$ and $nh^{2d} \to \infty$,
	\[\expect\left[   \pi_0(X)\abs{\overline{\tau}(X)} \ind(\op{sign}(\overline{\tau}(X)) \neq \op{sign}(\hat{\overline\tau}(X)))\right]  \to 0.
	\] 
	
	As a result, when $n\to \infty$, $h \to 0$ and $nh^{2d} \to \infty$, $\sup\limits_{\tau \in \mathcal{T}}\overline R_{\pi_0}(\hat{\pi}^*(\cdot;\Gamma);\Gamma) - \sup\limits_{\tau \in \mathcal{T}}\overline R_{\pi_0}(\pi^*(\cdot;\Gamma);\Gamma) \to 0$.
	
\end{proof}

\section{PCATE sensitivity bounds}\label{appendix: PCAT}
Analogously, the corresponding sensitivity bounds for partial conditional average treatment effect are:
	\begin{align}
	\overline{\tau}(x_S; \Gamma) = \overline{\mu}_1(x_S; \Gamma)  - \underline{\mu}_{0}(x_S; \Gamma), \label{eq: partial-cate-upper} \\  \underline{\tau}(x_S; \Gamma) = \underline{\mu}_1(x_S; \Gamma)  - \overline{\mu}_{0}(x_S; \Gamma), \label{eq: partial-cate-lower}
	\end{align}
	where $\overline{\mu}_{t}(x_S; \Gamma)$ and $\underline{\mu}_{t}(x_S; \Gamma)$ for $t \in \{0, 1\}$ are given in \eqref{eq: partial-outcome-upper}\eqref{eq: partial-outcome-lower}. 
The corresponding PCATE bounds estimators are:
\begin{align}
\hat{\overline{\tau}}(x_S; \Gamma) &= \hat{\overline{\mu}}_1(x_S; \Gamma)  - \hat{\underline{\mu}}_{0}(x_S; \Gamma) ,  \label{eq: est-partial-upper-cate} \\ \hat{\underline{\tau}}(x_S; \Gamma) &= \hat{\underline{\mu}}_1(x_S; \Gamma)  - \hat{\overline{\mu}}_{0}(x_S; \Gamma), \label{eq: est-partial-lower-cate}
\end{align}
where $\hat{\overline{\mu}}_{t}(x_S; \Gamma)$ and $\hat{\underline{\mu}}_{t}(x_S; \Gamma)$ for $t \in \{0, 1\}$ are given in \eqref{eq: est-partial-upper-outcome}\eqref{eq: est-partial-lower-outcome}. 

For any $\pi_S: \mathcal{X}_S \to [0, 1]$, the policy value and the worst-case policy regret are:
\begin{align*}
V(\pi_S;\tau) = \expect[\pi(X_S)Y(1) + (1 - \pi(X_S))Y(0)]
\end{align*}
\begin{equation}\label{eq: partial-regret}
	\overline R^S_{\pi_0}(\pi;\Gamma) = 
	\sup_{\tau(x_{S})\in\mathcal T(x_S;\Gamma)~\forall x_S\in\mathcal X_S}(V(\pi_S;\tau)  - V(\pi_0;\tau))
\end{equation}

\begin{corollary}\label{corollary: mismatch-x-population}
	Consider the partial conditional expected potential outcome 
	\[
		\mu_t(x_S)=\Efb{Y(t) \mid X_S=x_S},
	\]
	where $t \in \{0, 1\}$, $X_S$ is a subset of the observed covariates $X$, and $x_S \in \mathcal{X}_S$. The corresponding population PCAT sensitivity bounds \eqref{eq: partial-outcome-upper}\eqref{eq: partial-outcome-lower} have the following equivalent characterization: {\if1\forarxiv\small\fi
	\begin{align*}
	\overline{\mu}_t(x_S; \Gamma) &= 
	\sup\limits_{u \in \mathcal{U}^{nd} }
	\frac{\iint \alpha_t(x_S, x_{S^c}) yf_t(y, x_{S^c} \mid x_S)dydx_{S^c}+ \iint (\beta_t(x_S, x_{S^c}) - \alpha_t(x_S, x_{S^c}))u(y)yf_t(y, x_{S^c} \mid x_S)dydx_{S^c}}{\iint \alpha_t(x_S, x_{S^c})f_t(y, x_{S^c} \mid x_S)dydx_{S^c}+  \iint (\beta_t(x_S, x_{S^c}) - \alpha_t(x_S, x_{S^c}))u(y)f_t(y, x_{S^c} \mid x_S)dydx_{S^c}} \\
	\underline{\mu}_t(x_S; \Gamma) &= 
	\inf\limits_{u \in \mathcal{U}^{ni}}
	\frac{\iint \alpha_t(x_S, x_{S^c}) yf_t(y, x_{S^c} \mid x_S)dydx_{S^c}+ \iint (\beta_t(x_S, x_{S^c}) - \alpha_t(x_S, x_{S^c}))u(y)yf_t(y, x_{S^c} \mid x_S)dydx_{S^c}}{\iint \alpha_t(x_S, x_{S^c})f_t(y, x_{S^c} \mid x_S)dydx_{S^c}+  \iint (\beta_t(x_S, x_{S^c}) - \alpha_t(x_S, x_{S^c}))u(y)f_t(y, x_{S^c} \mid x_S)dydx_{S^c}} 
	\end{align*}}
	where $f_t(y, x_{S^c} \mid x_S)$ is the conditional joint density function for $\{T = t, Y(t), X_{S^c}\}$ given $X_S = x_S$ with $X_{S^c}$ as the complementary subset of $X$ with respect to $X_S$, $\alpha_t(\cdot)$ and $\beta_t(\cdot)$ are defined in \eqref{eq: uncertainty-set}, and $\mathcal{U}^{nd}$ and $\mathcal{U}^{ni}$ are defined in Lemma \ref{lemma: population-bounds}.
\end{corollary}

\begin{proof}
	By analogous arguments of change of variable and duality in the proof for Lemma \ref{lemma: population-bounds}, we can prove the conclusions in Corollary \ref{corollary: mismatch-x-population}. 
\end{proof}

\begin{corollary}\label{corollary: mismatch-est-consistency}
	Consider the following estimators: 
	\begin{align*}
	\hat{\overline{\mu}}_t(x_S; \Gamma) &= \sup\limits_{W_{ti} \in [\alpha_t(X_i; \Gamma), \beta_t(X_i; \Gamma)]} \frac{\sum_{i=1}^n\ind{(T_i =  t)}K(\frac{X_{i, S}-x_S}{h})W_{ti}Y_i}{\sum_{i=1}^n\ind{(T_i =  t)}K(\frac{X_{i, S}-x_S}{h})W_{ti}}, \\
	\hat{\underline{\mu}}_t(x_S; \Gamma) &= \inf\limits_{W_{ti} \in [\alpha_t(X_i; \Gamma), \beta_t(X_i; \Gamma)]}\frac{\sum_{i=1}^n\ind{(T_i =  t)}K(\frac{X_{i, S}-x_S}{h})W_{ti}Y_i}{\sum_{i=1}^n\ind{(T_i =  t)}K(\frac{X_{i, S}-x_S}{h})W_{ti}}.
	\end{align*}
	where $\alpha_t(\cdot)$ and $\beta_t(\cdot)$ are defined in \eqref{eq: uncertainty-set}. 
	
	Assume that $e_t(x_{S}, x_{S^c})$ and $f_t(y, x_{S^c} \mid x_S)$ are twice continuously differentiable with respect to $x_S$ for any $y \in \mathcal{Y}$ and $x_{S^c} \in \mathcal{X}_S$ with bounded first and second derivatives. Under the other assumptions in Theorem \ref{thm: consistent-cate}, when $n \to \infty$, $h \to 0$, and $nh^{2|S|} \to \infty$, $\hat{\overline{\mu}}_t(x_S) \overset{\text{p}}{\to} \overline{\mu}_t(x_S)$ and $\hat{\underline{\mu}}_t(x_S) \overset{\text{p}}{\to} \underline{\mu}_t(x_S)$.
\end{corollary}
\begin{proof}
	Following the proof for Theorem \ref{thm: consistent-cate}, we can analogously prove that when $n \to \infty$ and $nh^{2|S|} \to \infty$, {\if1\forarxiv\footnotesize\fi
	\begin{align}
		\hat{\overline{\mu}}_t(x_S; \Gamma) 
		&\overset{\text{p}}{\to} \sup\limits_{u \in \mathcal{U}^{nd}} 
			\frac{
							\expect \bigg[\ind{(T_i =  t)}\alpha_t(X_{i, S}, X_{i, S^c})K(\frac{X_{i, S}-x_S}{h})Y_i 
							+ \ind{(T_i =  t)}u(Y_i)(\beta_t(X_{i, S}, X_{i, S^c}) - \alpha_t(X_{i, S}, X_{i, S^c}))K(\frac{X_{i, S}-x_S}{h})Y_i\bigg]
					}
					{
						\expect \bigg[\ind{(T_i =  t)}\alpha_t(X_{i, S}, X_{i, S^c})K(\frac{X_{i, S}-x_S}{h}) 
							+ \ind{(T_i =  t)}u(Y_i)(\beta_t(X_{i, S}, X_{i, S^c}) - \alpha_t(X_{i, S}, X_{i, S^c}))K(\frac{X_{i, S}-x_S}{h})\bigg] \label{eq: mismatch-xS^cample}
					} 
	\end{align}}
	Note that 
	\[
	\expect \bigg[\ind{(T_i =  t)}\alpha_t(X_{i, S}, X_{i, S^c})K(\frac{X_{i, S}-x_S}{h})Y_i\bigg]  = \iiint \alpha_t(x'_S, x_{S^c})f_t(y, x'_S, x_{S^c})K(\frac{x'_S - x_S}{h})y dydx'_Sdx_{S^c}.
	\]
	
	By the similar Taylor expansion argument in the proof for Theorem \ref{thm: consistent-cate}, when $h \to 0$, 
	\[
	\expect \bigg[\ind{(T_i =  t)}\alpha_t(X_{i, S}, X_{i, S^c})K(\frac{X_{i, S}-x_S}{h})Y_i\bigg]  \to \iint \alpha_t(x_S, x_{S^c})f_t(y, x_S, x_{S^c})y dydx_{S^c}.
	\]
	
	Similarly, we can prove the convergence of other components in Corollary \ref{corollary: mismatch-est-consistency}. Given the characterization in  Corollary \ref{corollary: mismatch-x-population}, we can decompose the estimation bias in a way similar to \eqref{eq: error-decompose}, which leads to the final conclusions. 
\end{proof}

\begin{corollary} 
	Define the following policies based on the subset observed covariates $X_S$: for any $x_S \in \mathcal{X}$, 
	\begin{align*}
			\pi^P(x_S; \Gamma) &= \ind(\overline{\tau}(x_S; \Gamma) \le 0) +  \pi_0(x_S)\ind(\underline{\tau}(x_S; \Gamma) \le 0 < \overline{\tau}(x_S; \Gamma)) \\
			\hat{\pi}^P(x_S; \Gamma) &= \ind(\hat{\overline{\tau}}(x_S; \Gamma) \le 0) +  \pi_0(x_S)\ind(\hat{\underline{\tau}}(x_S; \Gamma) \le 0 < \hat{\overline{\tau}}(x_S; \Gamma)).
	\end{align*}
	where  $\overline{\tau}(x_S; \Gamma)$ and $\underline{\tau}(x_S; \Gamma)$ are the population PCATE sentivity bounds defined in \eqref{eq: partial-cate-upper}\eqref{eq: partial-cate-lower}, and $\hat{\overline{\tau}}(x_S; \Gamma)$ and $\hat{\underline{\tau}}(x_S; \Gamma)$ are the PCATE sensitivity bounds estimators given in \eqref{eq: est-partial-upper-cate} 
	\eqref{eq: est-partial-lower-cate}. 

	Then $\pi^P(\cdot ; \Gamma)$ is the population  minimax-optimal policies. Namely, 
	\begin{align*}
		\pi^P(\cdot; \Gamma)  &\in \argmin\limits_{\pi_S: \mathcal{X}_S \to [0, 1]} [\sup\limits_{\tau \in \mathcal{T}_S}\overline R^S_{\pi_0}(\pi;\Gamma)]
%`
	\end{align*}
	where $\mathcal{T}_S = \{\tau: \tau(x_S) \in [\underline{\tau}(x_S), \overline{\tau}(x_S)], \forall x_S \in \mathcal{X}_S \}$. Furthermore, the sample policy $\hat{\pi}^P$ is asymptotically minimax-optimal:  
	\[
		\overline R^S_{\pi_0}(\hat{\pi}^P(\cdot; \Gamma);\Gamma)  \overset{\text{p}}{\to} 	\overline R^S_{\pi_0}({\pi}^P(\cdot; \Gamma);\Gamma) .
	\]
\end{corollary}
\begin{proof}
	The conclusions can be proved analougously to the proofs for Proposition \ref{prop: population-opt-policy} and Theorem \ref{thm: policy}.
\end{proof}

\section{Additional figures} 
\begin{figure}[h!]
	%`
	\centering
	\includegraphics[width=0.5\textwidth]{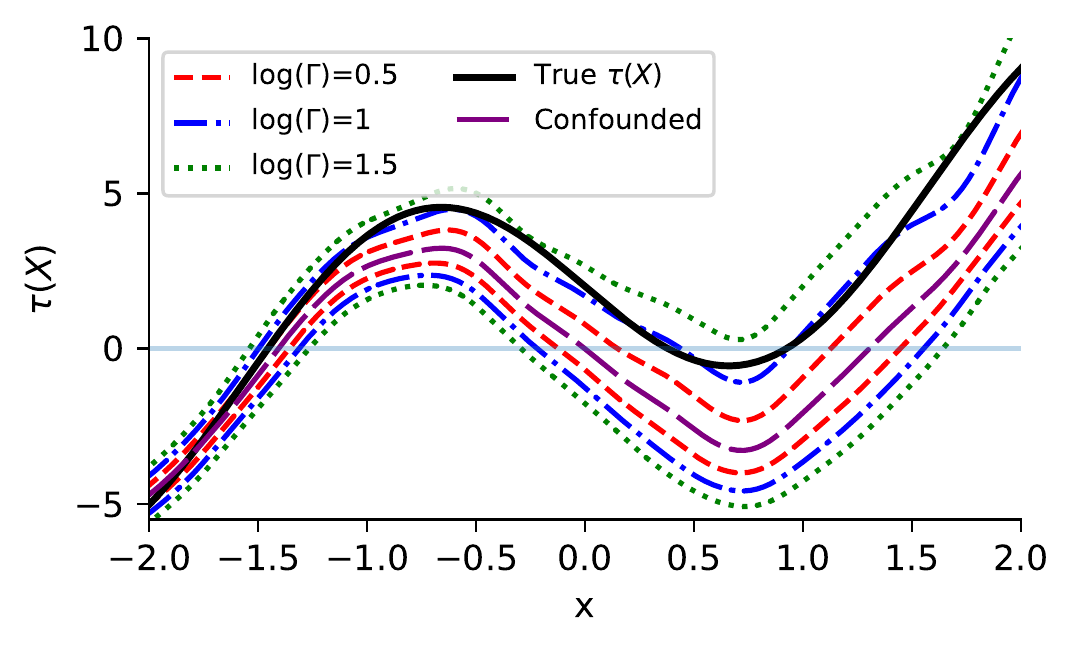}
	\caption{Bounds on CATE for differing values of $\Gamma$.}\label{fig-sinusoidal2000}
	%`
\end{figure}
In the main text, for the sake of illustration we presented an example where we fix the nominal propensities $e(X)$ and set the true propensities $e(X,U)$ such that the maximal odds-ratio bounds $\alpha(X), \beta(X)$ are ``achieved'' by the true propensities. We computed bounds on the pre-specified $e(X)$. We now consider a setting with a binary confounder where $\log(\Gamma)=1$ is not a uniform bound, but bounds most of the observed odds ratios, and instead we learn the marginal propensities $\Pr[T=1\mid X=x]$ from data using logistic regression. The results (Fig.~\ref{fig-sinusoidal2000}) are materially the same as in the main text. 

We consider the same setting as in Fig.~\ref{fig-1d} with a binary confounder $u \sim \op{Bern}(\nicefrac{1}{2})$ generated independently, and $X \sim \op{Unif}[-2,2]$. Then we set the true propensity score as
 $$e^*(x,u) = \sigma(\theta x + 2(  u-0.5) + 0.5)$$
We learn the nominal propensity scores $e(x)$ by predicting them from data with logistic regression, which essentially learns the marginalized propensity scores $e(x) = \Pr[T=1\mid X=x]$. The outcome model yields a nonlinear CATE, with linear confounding and with randomly generated mean-zero noise, $\epsilon \sim N(0,1)$: 
\[Y(t) = (2t-1) x + (2t-1) -  2\sin(2(2t-1) X) - 2 (u-1)(1 + 0.5X) + \epsilon\]

This outcome model specification yields a confounded CATE estimate of  
\begin{align*}&\E[Y\mid X=x, T=1] - \E[Y \mid X=x, T=0] \\
& = 2-2 x + 2( \sin(-2x)-\sin(2x)) +\\
& 2(2+x)( \Pr[u=1 \mid \substack{X=x\\T=1}] - \Pr[u=1\mid \substack{X=x\\T=0}] )
\end{align*}
By Bayes' rule, 
 $$ \Pr[u=1 \mid X=x, T=1] = \frac{\Pr[T=1\mid X=x, u=1] \Pr[u=1\mid X=x]}{\Pr[T=1\mid X=x]} $$
In Fig.~\ref{fig-sinusoidal2000}, we compute the bounds using our approach for $\log(\Gamma) = 0.5, 1, 1.5$ on a dataset with $n=2000$. The purple long-dashed line corresponds to a confounded kernel regression. (Bandwidths are estimated by leave-one-out cross-validation for each treatment arm regression). The confounding is greatest for large, positive $x$. The true, unconfounded CATE is plotted in black. While the confounded estimation suggests a large region, $x \in [0, 1.25]$, where $T=1$ is beneficial, the true CATE suggests a much smaller region where $T=1$ is optimal.

\end{appendices}
\end{document}